\documentclass[10pt]{llncs}

\usepackage{tikz}

\newcommand\copyrighttext{%
  \footnotesize This paper is a preprint of a paper submitted to and accepted for publication in
ICPR 2014 and is subject to IEEE copyright.}
\newcommand\copyrightnotice{%
\begin{tikzpicture}[remember picture,overlay]
\node[anchor=south,yshift=10pt] at (current page.south) {\fbox{\parbox{\dimexpr\textwidth-\fboxsep-\fboxrule\relax}{\copyrighttext}}};
\end{tikzpicture}%
}

\usepackage{amsmath} 
\usepackage{amsfonts}
\usepackage{nicefrac}
\usepackage{algorithm}
\usepackage{algpseudocode}

\usepackage[caption=false,font=footnotesize]{subfig}
\usepackage{graphicx}
\usepackage{wrapfig}

\algtext*{EndWhile}
\algtext*{EndIf}

\newcommand{\N}{\mathbb{N}}
\newcommand{\R}{\mathbb{R}}

\newcommand{\p}{\text{Pr}}
\newcommand{\Eta}{\text{H}}
\newcommand{\kld}{\text{KL}}
\newcommand{\Var}{\text{Var}}

\newcommand{\e}{e}
\newcommand{\where}{\,\vline~}

\DeclareMathOperator{\E}{E}
\DeclareMathOperator*{\argmax}{arg\,max}

\begin{document}

\mainmatter              

\title{A Theoretical and Experimental Comparison of the EM and SEM Algorithm}
\titlerunning{Comparison of the EM and SEM Algorithm}  

\author{Johannes Bl\"omer, Kathrin Bujna, and Daniel Kuntze}
\authorrunning{Bl\"omer et al.} 
\tocauthor{Johannes Bl\"omer, Kathrin Bujna, and Daniel Kuntze} 
\institute{
University of Paderborn, 33098 Paderborn, Germany,\\
\email{\{bloemer, kathrin.bujna, kuntze\}@uni-paderborn.de}
}

\maketitle              

\copyrightnotice

\begin{abstract}
In this paper we provide a new analysis of the SEM algorithm.
Unlike previous work, we focus on the analysis of a single run of the algorithm.
First, we discuss the  algorithm for general mixture distributions. 
Second, we consider Gaussian mixture models and show that with high probability the update equations of the EM algorithm and its stochastic variant are almost the same, given that the input set is sufficiently large.
Our experiments confirm that this still holds for a large number of successive update steps.
In particular, for Gaussian mixture models, we show that the stochastic variant runs nearly twice as fast.
\end{abstract}

\section{Introduction}
Training the parameters of probabilistic models to describe a given data set is a central task in the field of data mining and machine learning. 
The Expectation-Maximization (EM) algorithm \cite{dempster77} is a general scheme for finding maximum-likelihood solutions for this parameter estimation problem.
It is used when the observed data $X$ can be seen as incomplete and when the problem given the corresponding complete data $(X,Z)$ is easy to solve.
Formally, the parameter estimation problem is to maximize the likelihood function
$L(\theta|X)=p(X|\theta)=\int p(X,Z|\theta)\, dZ$
over the choice of model parameters $\theta$.
Starting with an initial set of model parameters, the EM algorithm (cf. Alg.~\ref{em_algo}) repeatedly performs two steps. 
The first step derives an optimal distribution for the hidden values.
The second step computes the new set of parameters by maximizing the expectation (over the hidden values) of the complete-data likelihood.
During the decades since its first presentation, a lot of work has been done to improve the EM algorithm \cite{mclachlan08}.
Most of these improvements deal with two major drawbacks. 
On the one hand, the convergence of the EM algorithm can be very slow.
On the other hand, the EM algorithm is prone to converge only to a local maximum, or even worse, to a saddle point of the likelihood function.

In this paper we analyze a probabilistic variant of the EM algorithm, known as the Stochastic EM or SEM algorithm \cite{celeux85} (cf. Alg.~\ref{sem_algo}), and its relation to the classical EM algorithm.
Instead of maximizing the expectation in the second step, the SEM algorithm uses the distribution from the first step to sample an assignment for the hidden values $Z$.
Afterwards, it maximizes the complete-data likelihood only for that fixed assignment.
Its inherent randomness technically allows the algorithm to escape from a saddle point or an undesired local maximum of the likelihood function.

\paragraph{Related Work}
The models generated by the SEM algorithm are studied in \cite{ip94}.
For mixtures of distributions from the exponential family, the author shows that the sequence of models generated by the SEM iterations is an ergodic Markov chain converging weakly to a stationary distribution over models.
Furthermore, it is shown that under appropriate assumptions the mean of this stationary distribution converges to the maximum-likelihood estimate.
However, the mean of the stationary distribution usually can not be obtained by a single run of the algorithm.
Instead, a large number of restarts is necessary to retrieve a reasonable approximation of the mean distribution.

Some initial experimental comparison between EM and SEM algorithm can be found in \cite{dias04}.
The algorithms are applied to two small one-dimensional data sets (containing $150$ and $174$ points) and use Gaussian mixtures with $3$ and $2$ components, respectively.
The authors evaluate the log-likelihood of the sequence of the produced solutions and the log-likelihood surface in the neighbourhood of these solutions.
Furthermore, they compare the final Gaussian mixture models returned by the algorithms.
Their results indicate that the SEM algorithm converges faster and more reliable than the EM algorithm.

\begin{figure}[tb]
\begin{minipage}{.49\textwidth}\hrule\vspace{1ex}
\textbf{EM}($X$, $\theta$)\hfill\emph{Algorithm~\ref{em_algo}}
\hrule
\begin{algorithmic}[0]
	\While{\textless\textit{condition}\textgreater}
		\State $q(Z) = p(Z|X,\theta)$
		\State maximize $\E_{z\sim q}\left[ \ln p(X,Z|\theta) \right]$ w.r.t. $\theta$
	\EndWhile
\end{algorithmic}
\hrule

\vspace{3ex}

\caption{Classical EM algorithm}\label{em_algo}
\end{minipage}
\ 
\begin{minipage}{.49\textwidth}
\hrule\vspace{1ex}
\textbf{SEM}($X$, $\theta$)\hfill\emph{Algorithm~\ref{sem_algo}}
\hrule
\begin{algorithmic}[0]
	\While{\textless\textit{condition}\textgreater}
		\State $q(Z) = p(Z|X,\theta)$
		\State draw $z \propto q(Z=z)$
		\State maximize $p(X,Z=z|\theta)$ w.r.t. $\theta$
	\EndWhile
\end{algorithmic}
\hrule
\caption{Stochastic EM algorithm}\label{sem_algo}
\end{minipage}
\end{figure}


\paragraph{Our Contribution}
Our results suggest that in most cases where the classical EM algorithm is applicable, one can use the SEM algorithm to obtain similar results more efficiently.
The previous theoretical analysis of the SEM algorithm \cite{ip94} does not give any guarantees for a single run of the SEM algorithm.
In this paper, we analyze a single run of the SEM algorithm in relation to the classical EM algorithm.
First, we show that mixture distributions are suitable candidates for the application of the SEM algorithm.
Second, for Gaussian mixture models and for sufficiently large input sets, we show that in a single update step  with high probability the algorithms yield very similar results.

Previous experimental comparisons \cite{dias04} consider only small examples and do not compare the solutions after single update steps.
In this paper, we present experiments which confirm that our theoretical results even hold for a large number of successive steps.
Moreover, we show that the simplified maximization step of the SEM algorithm leads to considerably better running times.

\section{Stochastic EM Algorithm for Gaussian Mixtures}
In this section we show that general mixture distributions are suitable candidates for the SEM algorithm.
Moreover, we show that for mixtures of Gaussians, the EM and SEM algorithm compute almost the same parameter updates.

\subsection{Mixture Distributions}\label{sem_mixture}
For a mixture of $K\in\bbbn$ component distributions, the parameter vector takes the form
$\theta=(w_1,\ldots,w_K,\theta_1,\ldots,\theta_K)$,
where $w_1,\ldots,w_K\in\bbbr$ with $\sum_{k=1}^Kw_k=1$ are the weights, and $\theta_k$ is the parameter vector for the $k$-th component.

Consider an observation $X=(x_1,\ldots,x_N)$ consisting of $N\in\bbbn$ independent draws from $\theta$.
Drawing a single observation $x_n$ can be described as a two step process.
First, choose one of the component distributions $\theta_k$ with probability proportional to its weight $w_k$. 
Second, draw $x_n$ from the chosen component distribution.
Then, the natural choice for the hidden values is the matrix $Z=(z_{nk})\in\{0,1\}^{N\times K}$ with $\sum_{k=1}^Kz_{nk}=1$ for $n=1,\ldots,N$ that indicates whether a component is responsible for an observation.
Using this definition, we get that the probability of a single observation is
$p(X=x_n|\theta)=\sum_Zp(Z|\theta)p(X=x_n|Z,\theta)=\sum_{k=1}^K w_k p(x_n|\theta_k)$.

Proposition~\ref{prop:mixture} yields the proper maximization step of the SEM algorithm.
\begin{proposition}\label{prop:mixture}
Let $\theta$ denote a parameter vector of a mixture model.
Then, for any fixed $X$ and $Z=(z_{nk})$ defined as above, the complete-data likelihood $p(X,Z|\theta)$ is maximized w.r.t. $\theta$ by setting
\[\textstyle w_k=\frac{1}{N}\sum_{n=1}^Nz_{nk}\quad\text{and}\quad\theta_k=\argmax_{\theta_k'} p(Y_k|\theta_k')\]
for $k=1,\ldots,K$, where $Y_k=\{x_n|z_{nk}=1\}$.
\end{proposition}
\begin{proof}
We denote the new model computed by the SEM algorithm in some fixed round by $\theta^{X,Z}$.
Since
\[\textstyle p(X,Z|\theta)=p(Z|\theta)\cdot p(X|Z,\theta),\]
it is sufficient to show that $\theta^{X,Z}$ maximizes $p(Z|\theta)$ as well as $p(X|Z,\theta)$.
Using that $p(z_{nk}=1|\theta)=w_k$, $z_{nk}\in\{0,1\}$, and $\sum_{k=1}^K z_{nk}=1$, we conclude
\[\textstyle p(Z|\theta)=\prod_{n=1}^N\sum_{k=1}^K z_{nk}w_k=\prod_{k=1}^Kw_k^{c_k},\]
where $c_k=\sum_{n=1}^Nz_{nk}$.
From this, we derive $\log p(Z|\theta)=\sum_{k=1}^K c_k\log(w_k)$.
Applying Lemma~\ref{lem:sumlog} yields that $\theta^{X,Z}$ maximizes $p(Z|\theta)$.

It remains to show that $\theta^{X,Z}$ maximizes $p(X|Z,\theta)$.
However, due to the independence of the observations, we have
\begin{align*}
\textstyle\ln p(X|Z,\theta)=\sum_{n=1}^N\ln p(x_n|Z,\theta)=\sum_{n=1}^N\sum_{k=1}^K z_{nk}\ln p(x_n|\theta_k)= \sum_{k=1}^K \ln p(Y_k|\theta_k).
\end{align*}
\end{proof}

\begin{lemma}\label{lem:sumlog}
Let $K,N,c_1,...,c_K\in\N$ with $\sum_{k=1}^Kc_k=N$ and denote  the standard $(K-1)$-simplex by $\textstyle S^{K-1}=\{(p_1,\ldots,p_K)\in\R^K_+\where\sum_{k=1}^Kp_k=1\}$.
Then, the function $f:S^{K-1}\rightarrow\R$ with
$f(p)=f(p_1,\ldots, p_K)=\sum_{k=1}^Kc_k\log(p_k)$
is maximized by $p=(\frac{c_1}{N},\ldots,\frac{c_K}{N})$.
\end{lemma}
\begin{proof}
Let $q=(q_1,\ldots,q_K)$ with $q_k=\frac{c_k}{N}$ for $k=1,\ldots,K$.
Instead of $f$, we maximize
\begin{align*}
\textstyle\frac{f(p)}{N}+\Eta(q)
&\textstyle=\sum_{k=1}^Kq_k\log p_k-\sum_{k=1}^Kq_k\log q_k \textstyle=\sum_{k=1}^Kq_k\log\frac{p_k}{q_k}=-\kld(q,p)
\end{align*}
where $\Eta(\cdot)$ denotes the entropy and $\kld(\cdot,\cdot)$ the Kullback-Leibler divergence.
The lemma follows from $\kld(q,p)\geq0$ and $\kld(q,p)=0$ if and only if $p=q$.
\end{proof}

\paragraph{Limitations.}\label{mixture_limitations}
For the maximization step to be well defined, we have to assure that for $k=1,\ldots,K$ it holds that $|Y_k|=\sum_{n=1}^N z_{nk}\geq\zeta$ for some model dependent threshold $\zeta\in\bbbn$ (cf. Sec.~\ref{sem_gmm}).
Otherwise, it might not be possible to determine the new parameters. 
In these cases, it might be a good idea to replace the under-determined component.

\subsection{Gaussian Mixtures}\label{sem_gmm}
For a mixture of $K\in\bbbn$ multivariate Gaussians over $\bbbr^D$ the parameters of the $k$-th component consist of a mean $\mu_k\in\bbbr^D$ and a covariance matrix $\Sigma_k\in\bbbr^{D\times D}$.
That is, the parameter vector of the model takes the form
$\theta=(w_1,\ldots,w_K,\mu_1,\ldots,\mu_K,\Sigma_1,\ldots,\Sigma_K)$
and $X=(x_1,\ldots,x_N)$ with $x_n\in\bbbr^D$.
Then, the classical EM algorithm in each step deterministically computes
\[ \textstyle w_k^{EM}=\frac{1}{N}\sum_{n=1}^Np_{nk}\ ,\qquad\mu_k^{EM}=\frac{\sum_{n=1}^Np_{nk}x_n}{\sum_{n=1}^Np_{nk}}\ ,\]
\begin{equation}\label{sem_gmm:em_update}
\textstyle \Sigma_k^{EM}=\frac{\sum_{n=1}^Np_{nk}(x_n-\mu_k^{EM})(x_n-\mu_k^{EM})^T}{\sum_{n=1}^Np_{nk}}\ ,
\end{equation}
for $k=1,\ldots,K$, where 
\begin{equation}\label{sem_gmm:pnk}
\textstyle p_{nk} = p(z_{nk}=1|X,\theta)=\E[z_{nk}|X,\theta]. 
\end{equation}

After sampling the indicator matrix $Z=(z_{nk})\in\{0,1\}^{N\times K}$, the SEM algorithm computes the following updates for $k=1,\ldots,K$:
\[ \textstyle w_k^{SEM}=\frac{1}{N} \sum_{n=1}^N z_{nk}\ ,\qquad\mu_k^{SEM}=\frac{\sum_{n=1}^Nz_{nk}x_n}{\sum_{n=1}^Nz_{nk}}\ ,\]
\begin{equation}\label{sem_gmm:sem_update}
\textstyle \Sigma_k^{SEM}=\frac{\sum_{n=1}^Nz_{nk}(x_n-\mu_k^{SEM})(x_n-\mu_k^{SEM})^T}{\sum_{n=1}^Nz_{nk}}\ ,
\end{equation}
where $\mu_k^{SEM}$ and $\Sigma_k^{SEM}$ are the maximum likelihood solution for the parameter estimation problem of a single Gaussian distribution with respect to observation $Y_k=\{x_n|z_{nk}=1\}$. 
In contrast to the EM algorithm, the parameters computed by the SEM algorithm are in fact random variables.
For a derivation of these update equations see Prop.~\ref{prop:mixture} and e.g. \cite{bishop06}.

As already observed in Sec.~\ref{mixture_limitations}, strictly speaking, we have to assure that for $k=1,\ldots,K$ it holds $|Y_k|\geq D+1$.
Otherwise, $\Sigma_k^{SEM}$ would not be symmetric positive definite.
Assuming that the observations are given in general linear position, assigning at least $D+1$ points to each Gaussian is sufficient.


\subsection{Proximity of the Update Equations}\label{sem_proximity}

In this section we give probabilistic bounds on the differences between the computations of the EM and SEM algorithm.
Before we state our main results, we provide some preliminary definitions.
For vectors $v\in\bbbr^D$ we denote the $i$-th coordinate of $v$ by $(v)_i$, while for matrices $M\in\bbbr^{D\times D}$ we denote the $(i,j)$-th entry of $M$ by $(M)_{ij}$.
We define the spread of the $d$-th coordinate of the input set by
\begin{equation}\label{eq:spread}
\textstyle \Delta_d=\max_n (x_n)_d-\min_n (x_n)_d\ .
\end{equation}
Since scaling the data set results in scaled differences of the mean and covariance updates, our bounds have to depend on the spread.
Furthermore, we define the overall responsibility of the $k$-th component by $r_k:=\sum_{n=1}^Np_{nk}$.
We measure the difference between the weight updates in terms of $E[w_k^{SEM}] = w_k^{EM} = \frac{r_k}{N}$.

Since the means are not translation invariant with respect to $X$, we measure the differences of the mean updates in terms of the standard deviation of a suitable chosen random variable. 
Consider the difference of the means $\mu_k^{SEM}-\mu_k^{EM} =  \frac{\sum_{k=1}^K z_{nk} (x_n - \mu_k^{EM})}{\sum_{n=1}^N z_{nk}}$.
Ignoring the normalization factor, one summand of the numerator is a random variable that corresponds to the contribution of $x_n$ to the difference. 
The variance of this contribution in the $d$-th coordinate is given by $p_{nk}(1-p_{nk})\left(x_n-\mu^{EM}_k\right)_d^2$.
Since the $z_{nk}$ are independent random variables, the variance of the overall contribution is given by $\sum_{n=1}^N p_{nk}(1-p_{nk})\left(x_n-\mu^{EM}_k\right)_d^2$.
Thus, we measure the difference between $\mu_k^{SEM}$ and $\mu_k^{EM}$ in the $d$-th coordinate in terms of the corresponding standard deviation
\begin{equation}\label{eq:taudef}
\textstyle \tau_{kd}=\sqrt{\sum_{n=1}^N p_{nk}(1-p_{nk})\left(x_n-\mu^{EM}_k\right)_d^2}\ .
\end{equation}

Analogously, for the difference in the $(i,j)$-th entry of the $k$-th covariance update we set \(\Upsilon_{kn}=(x_n-\mu_k^{EM})(x_n-\mu_k^{EM})^T\) and use
\begin{equation}\label{eq:rhodef}
\textstyle \rho_{kij}=\sqrt{\sum_{n=1}^Np_{nk}(1-p_{nk})\left(\Upsilon_{kn}-\Sigma_k^{EM}\right)_{ij}^2}\ .
\end{equation}


To bound the difference between the update formulas, we state three separate theorems.
For a fixed probability $\delta$, we give bounds on the proximity of the updates that can be guaranteed with probability at least $1-\delta$.
To derive a result for the complete update, one has to combine all three theorems using the union bound.
Recall, due to translation invariance, Thm.~\ref{thm:weights} measures the differences in terms of the expected value $E[w_k^{SEM}] = \frac{r_k}{N}$, while  Thm.~\ref{thm:means} and Thm.~\ref{thm:covariances} use the standard deviations $\tau_{ki}$ and $\rho_{kij}$.
The first theorem bounds the difference of the weight updates for one component.

\begin{theorem}[Proximity of weights]\label{thm:weights}
Let  $k\in\{1,\ldots,K\}$ and $2e^{-\nicefrac{r_k}{3}}\leq\delta\leq1$.
Then, for $\lambda_w=\sqrt{\frac{3\ln\nicefrac{2}{\delta}}{r_k}}$ with probability at least $1-\delta$ it holds 
\[ \left|w_k^{SEM} - w_k^{EM} \right| \leq \lambda_w w_k^{EM}\ .\]
\end{theorem}

Hence, given $\delta$, we know that with probability at least $1-\delta$ the proximity of the weights is $E[w_k^{SEM}]$ times a factor inverse proportional to $\sqrt{r_k}$ and poly-logarithmic in $\delta^{-1}$. 
In particular, the dependency on $\delta^{-1}$ is mild, and, as one would expect, the accuracy of our bound improves with growing $r_k$ due to the law of large numbers.
The following two theorems state similar results for the means and covariances.

The second theorem bounds the difference in a single coordinate of the mean updates.
Since one can not expect good estimates for the means if the weights are not well approximated, it depends on the proximity $\lambda_w$.

\begin{theorem}[Proximity of means]\label{thm:means}
Let $k\in\{1,\ldots,K\}$ and denote by $E_w$ the event that
\begin{equation}\label{thm:means:assumption} 
\textstyle \left|w_k^{SEM} - w_k^{EM} \right| \leq \lambda_w w_k^{EM}
\end{equation}
where $0<\lambda_w<1$. Then, for $i\in\{1,\ldots,D\}$ and $0<\delta<1$
\[ p\left( \left|\left(\mu_k^{SEM}-\mu_k^{EM}\right)_i\right|\leq\frac{\lambda_\mu}{1-\lambda_w}\cdot\frac{\tau_{k i}}{r_k}\ \bigg\vert\ E_w \right) \geq 1-\delta \]
for
\begin{equation}\nonumber
\lambda_\mu=\begin{cases}
\sqrt{2\e\ln\nicefrac{2}{\delta}}&\mbox{if }\frac{\tau_{ki}}{\Delta_i}\geq\frac{1}{\e}\sqrt{2\e\ln\nicefrac{2}{\delta}} \\[2ex]
\frac{2\Delta_i}{\tau_{ki}}\ln\nicefrac{2}{\delta}&\mbox{otherwise }
\end{cases}\ .
\end{equation}
\end{theorem}

If the probability of success is small enough, i.e. $\delta \geq 2e^{-(e\tau_{ki}^2)/(2\Delta_i^2)}$, we can use the first case and obtain a better proximity.
The third theorem is the equivalent of Thm.~\ref{thm:means} for a single entry of the covariance matrix.
The bound is the sum of two terms, the analogon of the bound for the means from Thm.~\ref{thm:means} and
the error that arises from the accuracy of the mean estimates.

\begin{theorem}[Proximity of covariances]\label{thm:covariances}
Let $k\in\{1,\ldots,K\}$, $i,j\in\{1,\ldots,D\}$ and denote by $E_{w,\lambda}$ the event that
\begin{equation}  \label{thm:covariances:assumption1}
\textstyle\left|w_k^{SEM} - w_k^{EM} \right| \leq \lambda_w w_k^{EM}
\end{equation}
where $0<\lambda_w<1$ and that for $\ell=i,j$ it holds
\begin{equation} \label{thm:covariances:assumption2}
\textstyle \left|\left(\mu_k^{SEM}-\mu_k^{EM}\right)_\ell\right|\leq\frac{\lambda_{\mu \ell}}{1-\lambda_w}\cdot\frac{\tau_{k\ell}}{r_k}
\end{equation}
where $\lambda_{\mu \ell}>0$. 
Then, for $0<\delta<1$
\begin{align*}
  p\bigg(  &\left|\left(\Sigma_k^{SEM}-\Sigma_k^{EM}\right)_{ij}\right| \\ 
&\leq\frac{\lambda_\Sigma}{1-\lambda_w}\cdot\frac{\rho_{kij}}{r_k}+\frac{\lambda_{\mu i}\lambda_{\mu j}}{(1-\lambda_w)^2}\cdot\frac{\tau_{ki}\tau_{kj}}{r_k^2} \bigg\vert E_{w,\lambda} \bigg) \geq 1-\delta 
\end{align*}
for
\begin{equation}\nonumber
\lambda_\Sigma=\begin{cases}
\sqrt{2\e\ln\nicefrac{2}{\delta}}&\mbox{if }\frac{\rho_{kij}}{\Delta_i\Delta_j}\geq\frac{1}{\e}\sqrt{2\e\ln\nicefrac{2}{\delta}}\\[2ex]
\frac{2\Delta_i\Delta_j}{\rho_{kij}}\ln\nicefrac{2}{\delta}&\mbox{otherwise }
\end{cases}\ .
\end{equation}
\end{theorem}

\subsection{Remarks}

\paragraph{Limitations.}
Generally speaking, it is not possible to approximate a mixture component with too small a weight.
Thus, we cannot expect bounds that do not use assumptions on the weights, which correspond to the $r_k = N\cdot w^{EM}_k$.
Since Thm.~\ref{thm:means} and Thm.~\ref{thm:covariances} both depend on $\frac{1}{1-\lambda_w}$, 
they become arbitrarily large for $\lambda_w$ near 1.
For instance, to ensure $\lambda_w\leq \frac{1}{2}$, according to Thm.~\ref{thm:weights} all weights have to be at least $\frac{12 \cdot \ln(\nicefrac{2}{\delta})}{N}$. 
Since usually $N\gg K$ and since the dependence on $\delta^{-1}$ is logarithmic, this is a rather mild assumption.

\paragraph{Running Time.}\label{runtime}
While in the EM algorithm each data point is involved in the update of all components, the SEM algorithm uses each data point only for the update of exactly one component.
Thus, depending on its concrete instantiation, the speedup between the SEM algorithm and the EM algorithm might be up to a factor of $K$.
Regarding the parameter estimation problem of Gaussian mixture models, we counted the number of multiplications\footnote{The number of additions and the number of multiplications are dominated by the same term, while the number of divisions is much smaller. 
Furthermore, we assume that $N\gg D$.}
in $\bbbr$ during a single iteration of both algorithms.
The overall running time of the EM algorithm is dominated by $2 K N D^2$, while for the SEM algorithm it is dominated by $K N D^2$.
Thus, there is at least a factor 2 speedup which is confirmed in our experiments (cf. Sec.~\ref{evaluation}).

\subsection{Proof of the proximity bounds}

To bound the difference of the weights, we use the following Chernoff bound.
A proof can be found e.g. in \cite{mitzenmacher05}.
\begin{lemma}\label{lem:wgt_chernoff}
Let $X_1,\ldots,X_n$ be independent random variables in $\{0,1\}$ and let $Y=\sum_{i=1}^n X_i$.
Then, for each $0\leq\lambda\leq1$ we have
\[\p\left(|Y-\E[Y]|\geq\lambda\cdot\E[Y]\right)\leq2e^{-\E[Y]\frac{\lambda^2}{3}}.\]
\end{lemma}

As a corollary we get the following lemma that bounds the difference for a given probability of occurrence.
\begin{lemma}\label{lem:wgt_delta}
Let $X_1,\ldots,X_n$ be independent random variables in $\{0,1\}$ and let $Y=\sum_{i=1}^n X_i$.
Then, for $2e^{-\frac{\E[Y]}{3}}\leq\delta\leq1$ and $\lambda=\sqrt{\frac{3\ln\nicefrac{2}{\delta}}{\E[Y]}}$ we have
\[\p\left(|Y-\E[Y]|\geq\lambda\cdot\E[Y]\right)\leq\delta.\]
\end{lemma}

Using Lemma~\ref{lem:wgt_delta}, we are able to prove Theorem~\ref{thm:weights}.
\begin{proof}[Proof of Theorem~\ref{thm:weights}]
Let $W=\sum_{n=1}^Nz_{nk}=N\cdot w_k^{SEM}$.
Then, \(\E[W]=\sum_{n=1}^Np_{nk}=r_k=N\cdot w_k^{EM}.\)
Thus, applying Lemma~\ref{lem:wgt_delta} yields
$N\cdot|w_k^{SEM}-w_k^{EM}|\geq N\cdot\lambda w_k^{EM}$ with probability at most $\delta$.
\end{proof}


To bound the difference of the means and covariances, we use the following Chernoff-type bound.
\begin{lemma}\label{lem:dev_chernoff}
Let $X_1,\ldots,X_n$ be discrete, independent random variables with $\E[X_i]=0$ and $|X_i|\leq C$ for some constant $C\geq 0$ and $i=1,\ldots,n$. 
Let $Y=\sum_{n=1}^N X_i$.
Then, for $\lambda\geq0$ it holds
\[\textstyle\p\left(|Y|\geq\lambda\sqrt{\Var(Y)}\right)\leq 2e^{\frac{-\lambda^2}{2e^a}},\]
where $\Var(Y)=\sum_{i=1}^n\Var(X_i)$ and $a\geq0$ such that $\lambda=\frac{ae^a\sqrt{\Var(Y)}}{C}$.
\end{lemma}
\begin{proof}
Due to symmetry, we only prove $\p\left(Y\geq\lambda\sqrt{\Var(Y)}\right) \leq e^{ -\lambda^2/(2e^a) }$.
By Markov's inequality we obtain  
$\p\left(Y\geq\lambda\sqrt{\Var(Y)}\right)\leq \frac{E\left[e^{tY}\right]}{e^{t\lambda\sqrt{\Var(Y)}}}$ for each $t>0$.

Let $x_{ij}$ be the possible outcomes of $X_i$ with $\p(X_i = x_{ij})=p_{ij}$.
Then we obtain
\begin{align*}
\textstyle E\left[e^{tX_i}\right] 
&\textstyle= 1+\sum_{j=1}^m p_{ij} \left( \frac{1}{2!}(tx_{ij})^2 + \frac{1}{3!}(tx_{ij})^3 + \ldots \right)\\
&\textstyle= 1+\sum_{j=1}^m p_{ij} (tx_{ij})^2 \left( \frac{1}{2!} + \frac{1}{3!}(tx_{ij}) + \ldots \right)\\
&\textstyle\leq 1+\sum_{j=1}^m p_{ij} (tx_{ij})^2 \frac{1}{2}e^{tx_{ij}}  \leq 1+\frac{e^a}{2}t^2\Var(X_i)
\end{align*}
for $a \geq t C$.
Using $1+\alpha\leq e^\alpha$ for $\alpha\geq 0$, we get
\begin{align*}
\textstyle E\left[\e^{tY}\right] &\textstyle = \prod_{i=1}^n\E[e^{tX_i}] 
\leq \prod_{i=1}^n (1+\frac{e^a}{2} t^2\Var(X_i)) \\
\textstyle &\leq \prod_{i=1}^n e^{\frac{e^a}{2} t^2 \Var(X_i)}=  e^{\frac{e^a}{2} t^2 \Var(Y)}
\end{align*}
Setting $t=\frac{\lambda}{e^a \sqrt{\Var(Y)}}$ yields the claim.
\end{proof}

\begin{lemma}\label{lem:dev_delta}
Let $X_1,\ldots,X_n$ be discrete, independent random variables with $\E[X_i]=0$ and $|X_i|\leq C$ for some constant $C\geq0$ and $i=1,\ldots,n$. 
Let $Y=\sum_{n=1}^N X_i$ and $0<\delta<1$.
Then,
\[\p\left(|Y|\geq\lambda\sqrt{\Var(Y)}\right)\leq\delta\]
for
\begin{equation}\nonumber
\lambda=\begin{cases}
\sqrt{2\e\ln\nicefrac{2}{\delta}}&\mbox{if }\frac{\sqrt{\Var(Y)}}{C}\geq\frac{1}{\e}\sqrt{2\e\ln\nicefrac{2}{\delta}}\\[2ex]
\frac{2C}{\sqrt{\Var(Y)}}\ln\nicefrac{2}{\delta}&\mbox{otherwise}
\end{cases}
\end{equation}
\end{lemma}
\begin{proof}
We start with the first case, i.e. $\frac{\sqrt{\Var(Y)}}{C}\geq\frac{1}{\e}\sqrt{2\e\ln\nicefrac{2}{\delta}}$.
Let $\lambda:=\sqrt{2\e\ln\nicefrac{2}{\delta}}$ and choose $0\leq a\leq1$, such that $\lambda=\frac{ae^a\sqrt{\Var(Y)}}{C}$.
Applying Lemma~\ref{lem:dev_chernoff}, we obtain
\begin{align*}
\textstyle \p\left(|Y|\geq\lambda\sqrt{\Var(Y)}\right)
\leq2\exp\left(\frac{-\lambda^2}{2e^a}\right)
=2\exp\left(\frac{-2\e\ln\nicefrac{2}{\delta}}{2e^a}\right)
=2\left(\frac{\delta}{2}\right)^{\frac{\e}{\e^a}} 
\leq \delta.
\end{align*}

It remains to consider the second case, i.e. $\frac{\sqrt{\Var(Y)}}{C}<\frac{1}{\e}\sqrt{2\e\ln\nicefrac{2}{\delta}}$.
We define $\lambda:=\frac{2C}{\sqrt{\Var(Y)}}
\ln\nicefrac{2}{\delta}$ and let $a\in\R$, such that $\lambda=\frac{ae^a\sqrt{\Var(Y)}}{C}$.
It follows
\begin{equation}\label{eq:dev_epsilon}\textstyle
a\e^a=\frac{2C^2}{\Var(Y)}
\ln\nicefrac{2}{\delta}.
\end{equation}
Using $\frac{C}{\sqrt{\Var(Y)}}>\frac{\e}{\sqrt{2\e\ln\nicefrac{2}{\delta}}}$, we deduce
$a\e^a>\frac{2\e^2}{2\e\ln\nicefrac{2}{\delta}}
\ln\nicefrac{2}{\delta}=\e
$
and thus, $a>1$.
Applying Lemma~\ref{lem:dev_chernoff}, we obtain
\begin{align*}
\p\left(|Y|\geq\lambda\sqrt{\Var(Y)}\right)\leq&\textstyle2\exp\left(\frac{-\lambda^2}{2e^a}\right)=2\exp\left(-\frac{4C^2}{\Var(Y)}\left( \ln\nicefrac{2}{\delta}\right)^2\frac{1}{2e^a}\right).
\end{align*}
Using Equation~\eqref{eq:dev_epsilon} and $a>1$ we deduce
\begin{align*}
\textstyle\p\left(|Y|\geq\lambda\sqrt{\Var(Y)}\right)
&\leq\exp\left(-2a\e^a
\ln\nicefrac{2}{\delta}\frac{1}{2e^a}\right)=2\exp\left(-a
\ln\nicefrac{2}{\delta}\right)\\
&\textstyle<\textstyle2\exp\left(-
\ln\nicefrac{2}{\delta}\right)=2\left(\frac{\delta}{2}\right)
=\delta
\end{align*}
\end{proof}

\begin{proof}[Proof of Theorem~\ref{thm:means}]
Let $k\in\{1,\ldots,K\}$, $i\in\{1,\ldots,D\}$ and define the real random variable
\[\textstyle M_{kin}=(z_{nk}-p_{nk})\left(x_n-\mu_k^{EM}\right)_i.\]
Since $\E[z_{nk}]=p_{nk}$, we get that $\E\left[M_{kin}\right]=0$ and
\[\textstyle\Var(M_{kin})=p_{nk}(1-p_{nk})(x_n-\mu^{EM}_k)_i^2\]
Furthermore, since each $\mu_k^{EM}$ is a convex combination of $x_1,\ldots,x_N$, we get
\[\textstyle|M_{kin}|\leq|z_{nk}-p_{nk}|\cdot\left|\left(x_{n}-\mu_{k}^{EM}\right)_i\right|\leq\Delta_i.\]

Note that by definition of $\mu_k^{EM}$ we have
$\sum_{n=1}^N p_{nk}\left(x_n-\mu_k^{EM}\right)_i=0$.
Thus, for the random variable $M_{ki}= \sum_{n=1}^N M_{kin}$ it holds
\[\textstyle M_{ki}= \sum_{n=1}^N z_{nk}\left(x_n-\mu_k^{EM}\right)_i.\]
Furthermore, we get $\E[M_{ki}]=0$ and
\[\textstyle\Var(M_{ki})=\sum_{n=1}^N p_{nk}(1-p_{nk})(x_n-\mu^{EM}_k)_i^2=\tau^2_{ki}.\]

Applying Lemma \ref{lem:dev_delta} with $C=\Delta_i$ and the appropriate choice of $\lambda_\mu$ yields
\[\textstyle\p\left(\left|\sum_{n=1}^Nz_{nk}\left(x_n-\mu_k^{EM}\right)_i\right|\geq\lambda_\mu\cdot\tau_{ki}\right)\leq\delta.\]

Using Assumption \eqref{thm:means:assumption}, we conclude that
\begin{align*}
\textstyle\left|\left(\mu_k^{SEM}-\mu_k^{EM}\right)_i\right|
&\textstyle=\left|\frac{\left(\sum_{n=1}^Nz_{nk}x_n-\sum_{n=1}^Nz_{nk}\mu_k^{EM}\right)_i}{\sum_{n=1}^Nz_{nk}}\right|\\
&\textstyle=\left|\frac{\sum_{n=1}^Nz_{nk}\left(x_n-\mu_k^{EM}\right)_i}{N\cdot w_k^{SEM}}\right|
\leq\frac{\lambda_\mu\tau_{ki}}{N\cdot(1-\lambda_w)w_k^{EM}}=\frac{\lambda_\mu\tau_{ki}}{(1-\lambda_w)r_k}
\end{align*}
with probability at least $1-\delta$.
\end{proof}


\begin{proof}[Proof of Theorem~\ref{thm:covariances}]
Let $\nu=\mu_{k}^{SEM}-\mu_{k}^{EM}$.
By using the update formulas of the means, we get
\[ \textstyle \Sigma_k^{SEM}+\nu\nu^T=\frac{\sum_{n=1}^Nz_{nk}(x_n-\mu_k^{EM})(x_n-\mu_k^{EM})^T}{\sum_{n=1}^Nz_{nk}}.\]

Analogously to the proof of Theorem~\ref{thm:means}, we can bound the distance between $\Sigma_k^{SEM}+\nu\nu^T$ and $\Sigma_k^{EM}$.
To this end, we define a real random variable 
\[\textstyle S_{kijn}=(z_{nk}-p_{nk})\left(\Upsilon_{kn}-\Sigma_k^{EM}\right)_{ij}\]
where $\Upsilon_{kn} = (x_n - \mu_k^{EM})(x_n-\mu_k^{EM})^T$.
Again, using the definitions, we obtain $\E[S_{kijn}]=0$, $\Var(S_{kijn})=p_{nk}(1-p_{nk})\left(\Upsilon_{kn}-\Sigma_k^{EM}\right)_{ij}^2$, and $|S_{kijn}|\leq \Delta_i\Delta_j$.

Then, for the sum $S_{kij}=\sum_{n=1}^N S_{kijn}$ it follows analogously to the means that
\[ \textstyle S_{kij} =  \sum_{n=1}^N z_{nk}\left(\Upsilon_{kn}-\Sigma_k^{EM}\right)_{ij}.\]
Moreover, we get $\E[S_{kij}]=0$ and $\Var(S_{kij}) = \rho^2_{kij}$.

Applying Lemma \ref{lem:dev_delta} with $C=\Delta_i\Delta_j$ and the appropriate choice of $\lambda_\Sigma$ yields
\[\textstyle\p\left(\left|\sum_{n=1}^N z_{nk}\left(\Upsilon_{kn}-\Sigma_k^{EM}\right)_{ij}\right|\geq\lambda_\Sigma\cdot\rho_{kij}\right)\leq\delta.\]

Furthermore, using the triangle inequality, we get
\begin{align*}
&\textstyle\left|\left(\Sigma_k^{SEM}-\Sigma_k^{EM}\right)_{ij}\right|
=\left|\left(\frac{\sum_{n=1}^Nz_{nk}\Upsilon_{kn}}{\sum_{n=1}^Nz_{nk}}-\nu\nu^T-\Sigma_k^{EM}\right)_{ij}\right|\\
\leq&\textstyle\left|\left(\frac{\sum_{n=1}^Nz_{nk}\left(\Upsilon_{kn}-\Sigma_k^{EM}\right)}{\sum_{n=1}^Nz_{nk}}\right)_{ij}\right|
+\left|\left(\mu_{k}^{SEM}-\mu_k^{EM}\right)_i\right|\cdot\left|\left(\mu_k^{SEM}-\mu_k^{EM}\right)_j\right|\\
\end{align*}

Using Assumptions \eqref{thm:covariances:assumption1} and \eqref{thm:covariances:assumption2}, we obtain that with probability at least $1-\delta$ it holds
\[\textstyle\left|\left(\Sigma_k^{SEM}-\Sigma_k^{EM}\right)_{ij}\right|\leq\frac{\lambda_\Sigma}{(1-\lambda_w)}\cdot\frac{\rho_{kij}}{r_k}+\frac{\lambda_{\mu i}\lambda_{\mu j}}{(1-\lambda_w)^2}\cdot\frac{\tau_{ki}\tau_{kj}}{r_k^2}.\]
\end{proof}

\section{Evaluation}\label{evaluation}
To underpin our proximity and running time analysis, we compare the computations of the classical EM and the SEM algorithm on different input sets. 

\subsection{Implementation and Data Sets}

\paragraph{Implementation.}
To get comparable results, both algorithms were implemented from scratch in C++ using the linear algebra library Eigen. 
Furthermore, as mentioned at the end of Sec.~\ref{sem_gmm}, it is possible that the number of points sampled to a component is too small to compute a new covariance.
For empty components we solve the problem by sampling a new mean and computing the covariance matrix from scratch as in \cite{dasgupta07}. 
In case of too few points, we try to mix the under-determined covariance with the previous covariance matrix or 
simply keep the old covariance matrix.
For the EM algorithm we implemented a similar error handling since it has to deal with similar problems.
In our experiments these problems hardly ever occurred.

\paragraph{Data Sets}
We used artificial as well as real world input data.
For the generation of the artificial data sets, we considered different combinations of the dimension $D\in\bbbn$ and the number of component distributions $K\in\bbbn$.
For each combination, we probabilistically computed several parameter vectors $\theta$.
Then, for each $\theta$ and several $N\in\bbbn$ we drew $N$ points from the Gaussian mixture given by $\theta$.
To get reasonable results, it is important that the Gaussians are not pairwise well separated.
Otherwise, the task of learning the parameters is too easy.
Therefore, we ensured that the mixtures mainly consist of interfusing Gaussians.
Furthermore, we created each mixture with balanced as well as unbalanced weights.

\begin{figure}[t]
\centering
	\includegraphics[width=.5\columnwidth]{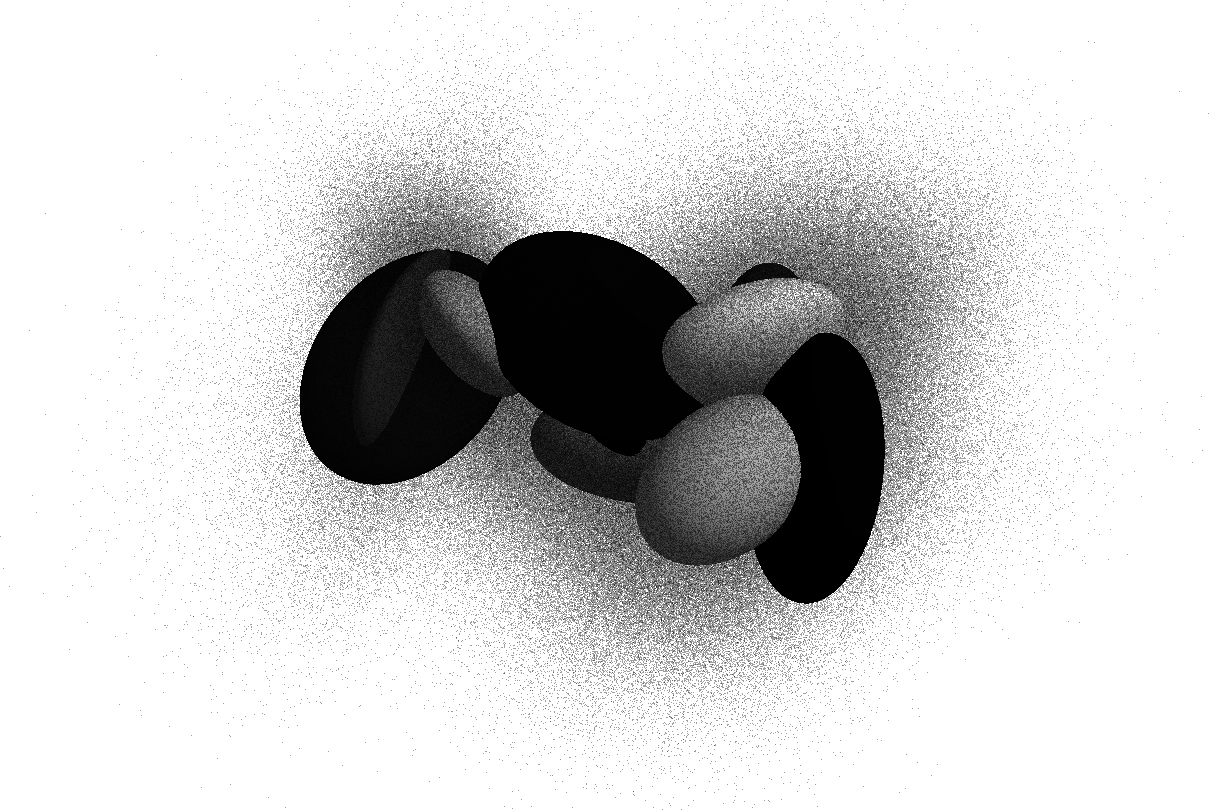}
\caption{Random projection of the artificial data set with $N=1\,000\,000$. The ellipsoids depict the standard deviation areas of each component distribution (the larger the weight of the component, the brighter the shade of the corresponding ellipsoid).}
\label{art_screenshot}
\end{figure} 

\begin{figure}[t]
\centering
\subfloat[Forest Covertype data set]{
		\label{fig:forest_screenshot}		
		\includegraphics[width=.35\columnwidth]{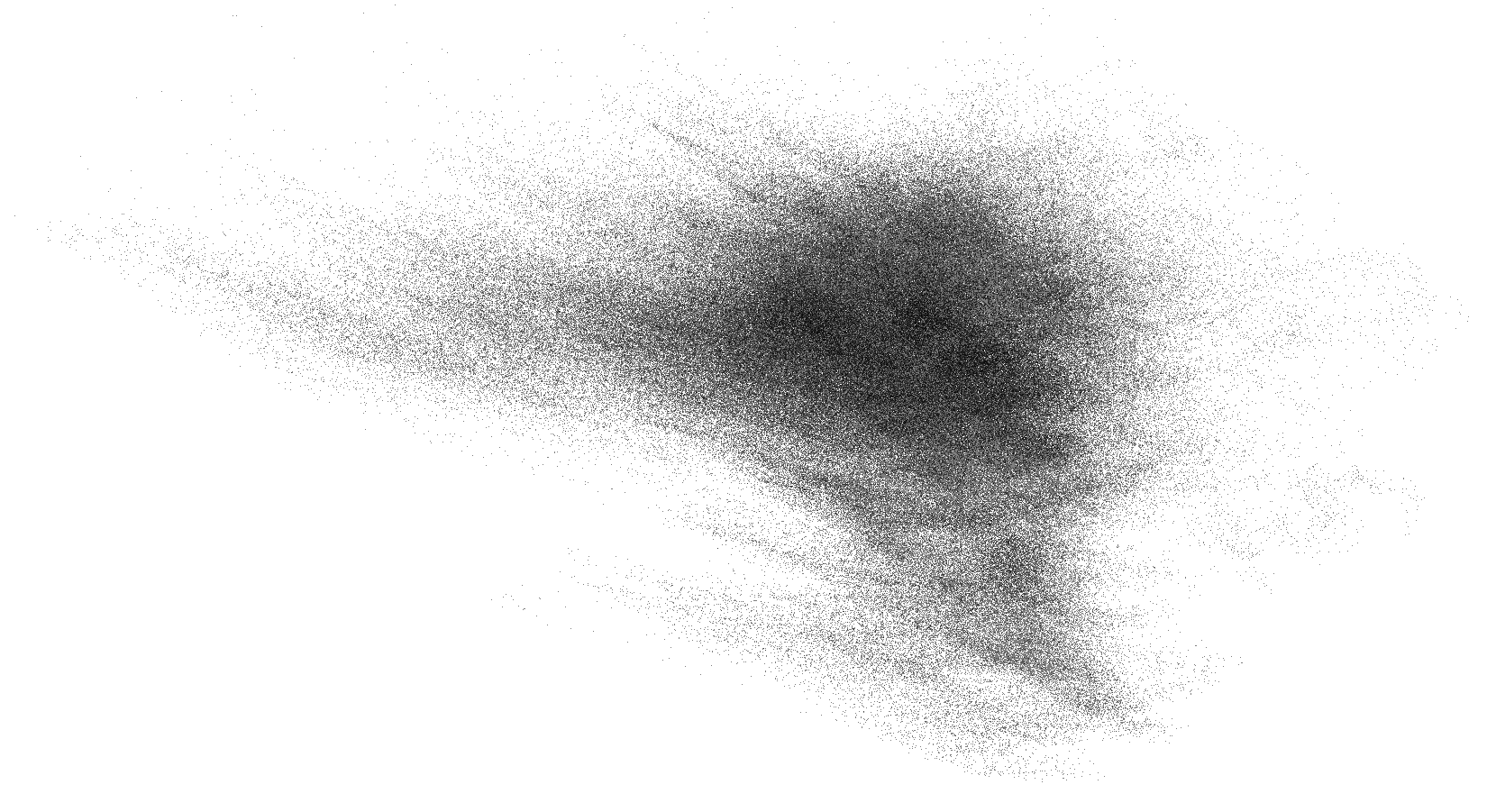}
}
\subfloat[ALOI features data set]{
		\label{fig:aloi_screenshot}
		\includegraphics[width=.35\columnwidth]{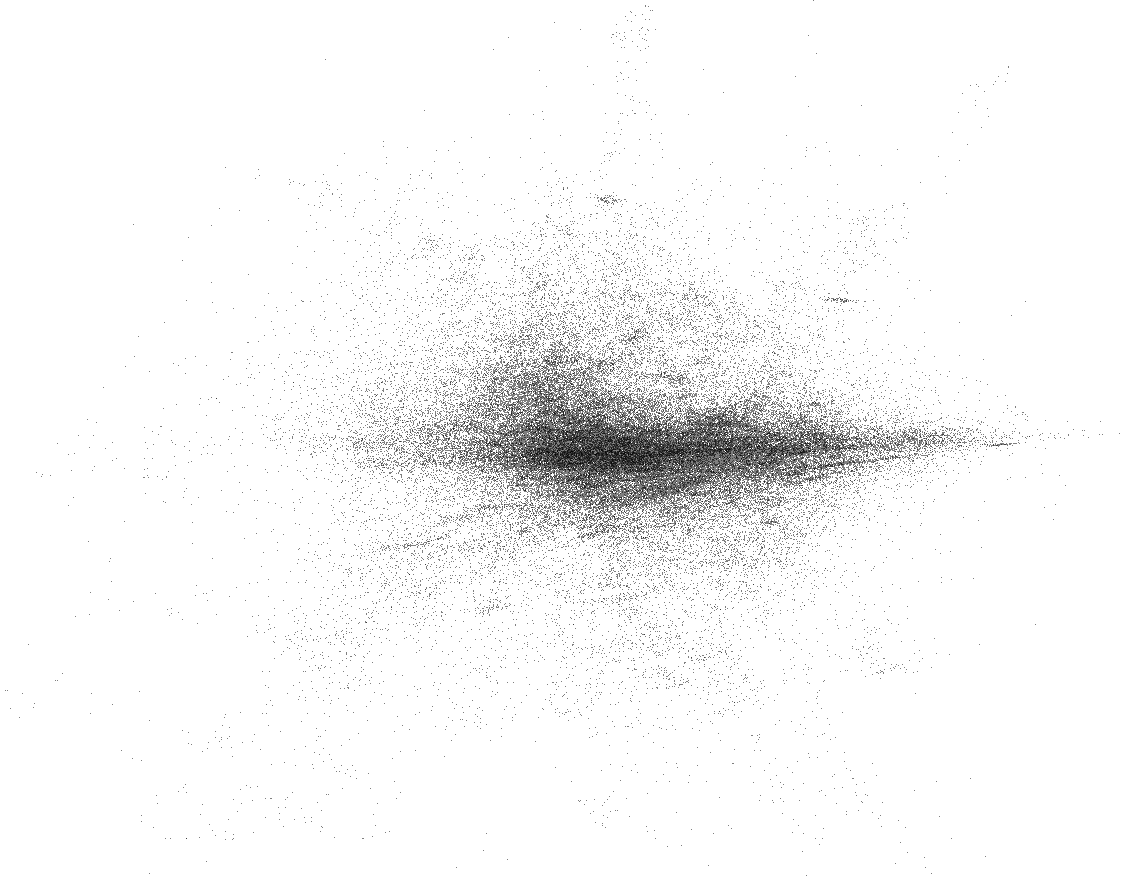}
}
\caption{Random projections of the real world data sets.}
\label{real_screenshot}
\end{figure}


As real world data we use two publicly available data sets. 
The first one is the \emph{Forest Covertype data set} which is part of the \emph{UCI Machine Learning Library}\cite{asuncion07} and contains 581\,012 data points.
To get data suitable for Gaussian mixture parameter estimation, we ignored the class labels and 44 qualitative binary attributes.
The second data set is based on the \emph{Amsterdam Library of Object Images (ALOI)}\cite{geusebroek05,kriegel12}.
This database consists of 110\,250 images of 1\,000 small objects, taken under various conditions.
We use a 27-dimensional feature vector set that is based on color histograms in HSV color space.
The features were extracted from the database as described in \cite{kriegel11} but using 2 bins for hue, saturation and brightness each.
Both real world data sets were normalized before use.
That is, for each coordinate of the input data points the values were translated and scaled to fit into the interval $[0,1]$. 
Otherwise, the difference between the means and covariances might be dominated by few dimensions of the input space, and the problem might reduce to a lower dimensional problem.

\subsection{Experiments}\label{subsec:experiments}

For the comparison of the EM and SEM algorithm, we established three different types of tests.
In the first type of tests we compared the negative log-likelihood of the parameters computed by the algorithms.
The goal of the second type of tests was to compare the intermediate solutions of the two algorithms during their execution.
To do this, we ran both algorithms for 50 rounds and computed the differences between the solutions after each round.

The aim of the third type of tests was to evaluate our theoretical bounds. 
For the sake of simplicity, we only compared the means computed by the two algorithms.
In each round of the SEM algorithm we also computed the mean updates as the EM algorithm would compute it given the current model parameters.
Then, we determined the actual Euclidean distance between each pair of means and our theoretical bounds. 
To get a bound that holds with high probability, we chose $\delta=\frac{1}{100\cdot K(D+1)}$.
Applying Them.~\ref{thm:weights} for all $K$ weight updates and Thm.~\ref{thm:means} for all $D$ coordinates in all $K$ mean updates yields $K(D+1)$ bounds that hold with probability $1-\delta$.
Using the union bound, the resulting bound for the Euclidean distance between the means holds with probability at least $1-\frac{1}{100}$.

Both algorithms need to be fed with an initial solution.
Thus, for each data set we created 30 sets of \emph{initial model parameters}.
For each initial parameter vector we proceeded as proposed in \cite{dasgupta07}.
That is, we drew the $K$ means $\mu_k$ uniformly at random from the input $X$, set the covariances to $\Sigma_k=I_D\cdot\frac{1}{2D}\min_{i\neq k}\|\mu_k-\mu_i\|^2$, and assigned the weight $\frac{1}{K}$ to each component.
To get comparable results, we always started both algorithms with the same initial model parameters.
Due to the randomization, for each initial model we perform 100 different runs and study the average behavior of the algorithms given a fixed initial solution.

\subsection{Results}\label{subsec:experiments_results}

Regarding the artificial data sets, the experiments for the different combinations of $D$ and $K$ led to essentially similar results.
Therefore, in the following we only discuss two particular artificial data sets with 
$N=1\,000\,000$ and $N=10\,000$, which where both generated from the same parameter vector with $D=K=10$.
Furthermore, we only depict results for some selected initial solutions since we observed that the respective results hardly differed from each other.

\paragraph{Log-Likelihood.} Some characteristic results of our first type of tests are shown in Fig.~\ref{fig:cost_gen}.
Similar to box plots, the dark gray line marks the median, the gray ribbon ranges from the lower to the upper quartile, while the light gray ribbon ranges from the minimum to the maximum of all data. 
For the majority of our experiments the log-likelihood of the solutions almost coincide (cf. Fig.~\ref{fig:gen1000000_I2}). 
Generally speaking, for small values of $N$ one can not expect that the SEM algorithm yields parameter estimates close those of the EM.
This is due to the law of large numbers, i.\,e. the influence of a single sampling step of the SEM algorithm is larger for smaller $N$.
Indeed, for the artificial data set with $N=10\,000$ we sometimes observe a small difference between the log-likelihood of the computed solutions as depicted in Fig.~\ref{fig:gen10000_I25}.
In some cases, this leads to final solutions with a log-likelihood that differs from those of the EM algorithm. 
As shown in Fig.~\ref{fig:gen10000_I13} and Fig.~\ref{fig:gen10000_I19}, the log-likelihood may be better or worse.

\begin{figure}
\centering
\subfloat[$N=1\,000\,000$]{
		\label{fig:gen1000000_I2}		
		\includegraphics[width=0.23\textwidth]{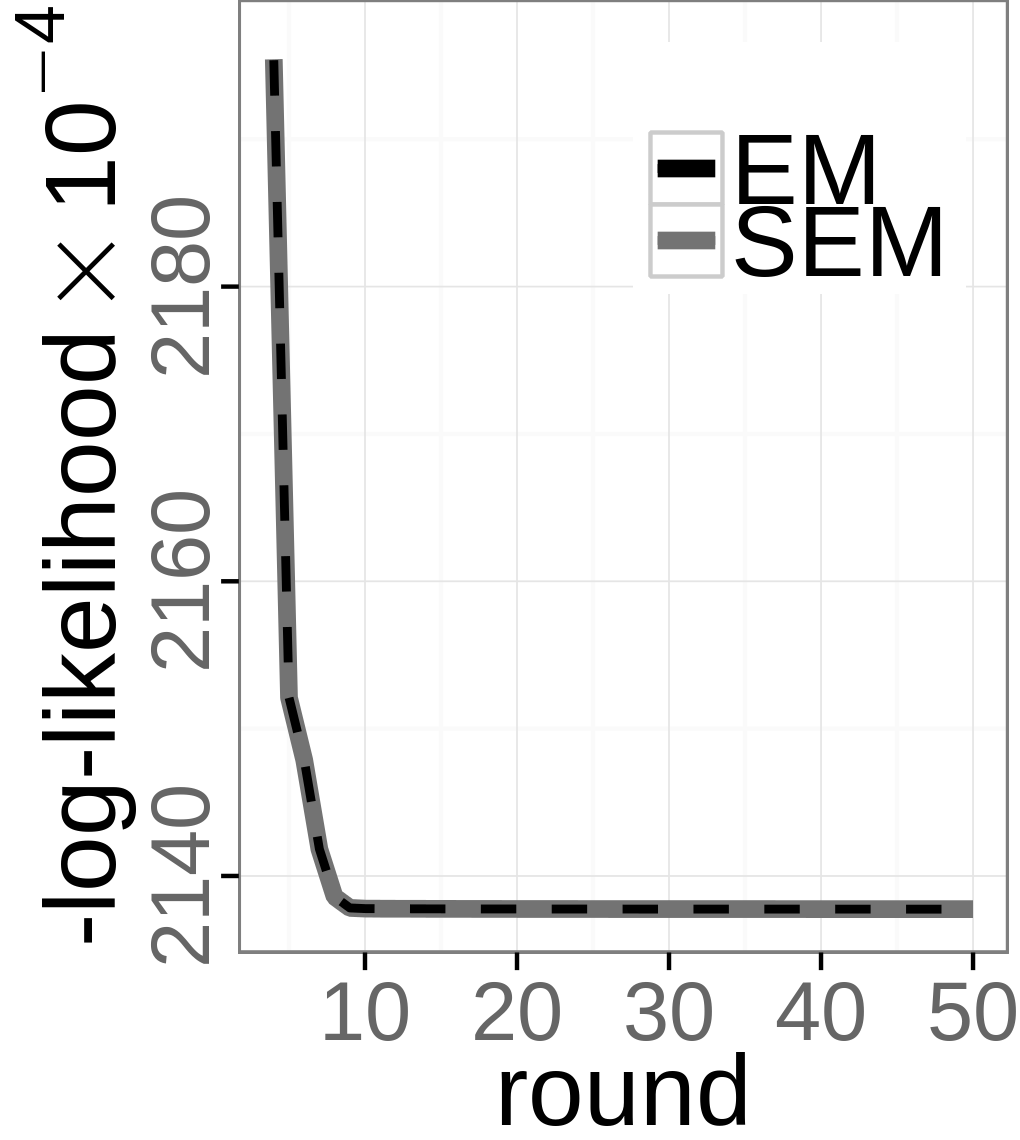}
}
\subfloat[$N=10\,000$]{
		\label{fig:gen10000_I25}		
		\includegraphics[width=0.23\textwidth]{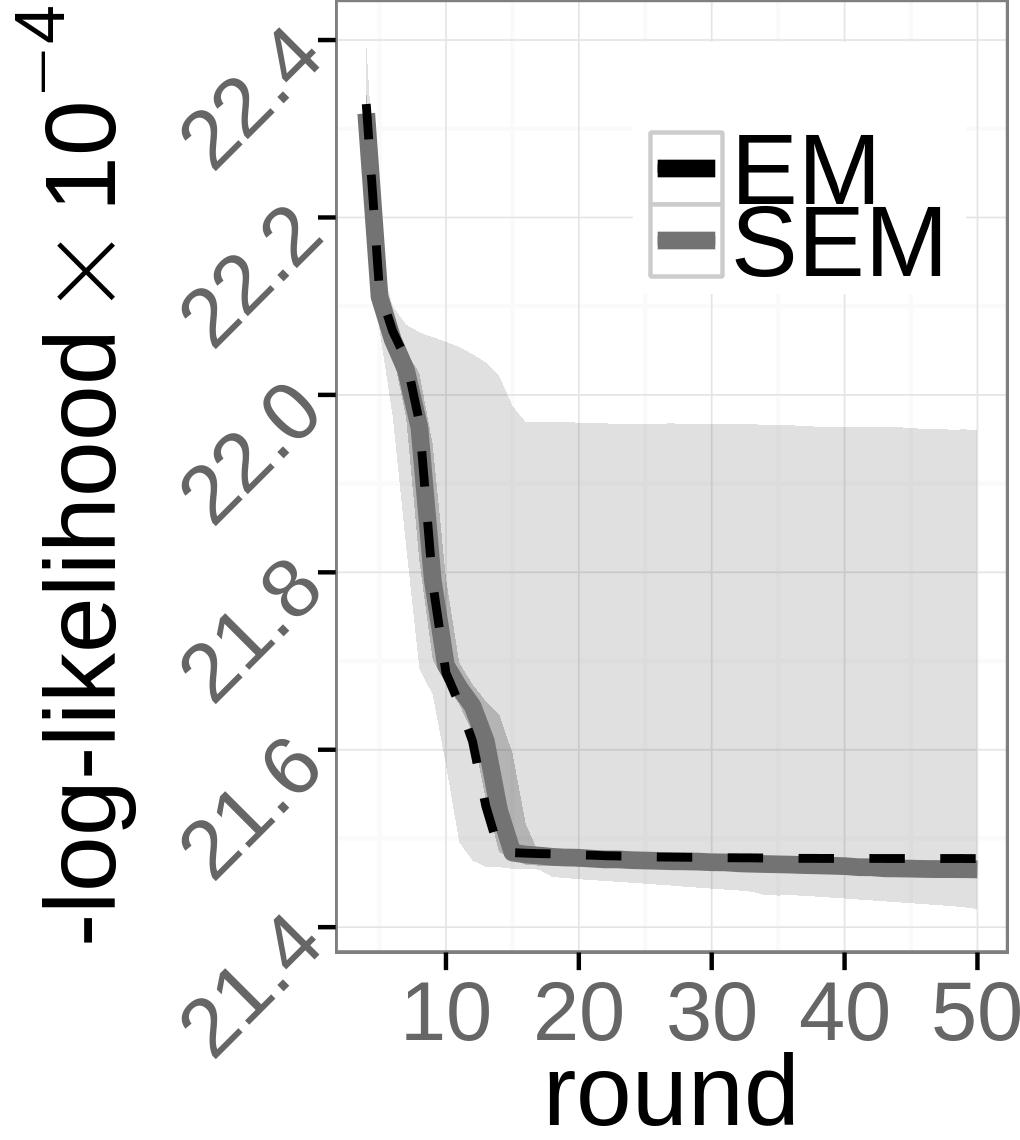}
}
\subfloat[$N=10\,000$]{
		\label{fig:gen10000_I13}		
		\includegraphics[width=0.23\textwidth]{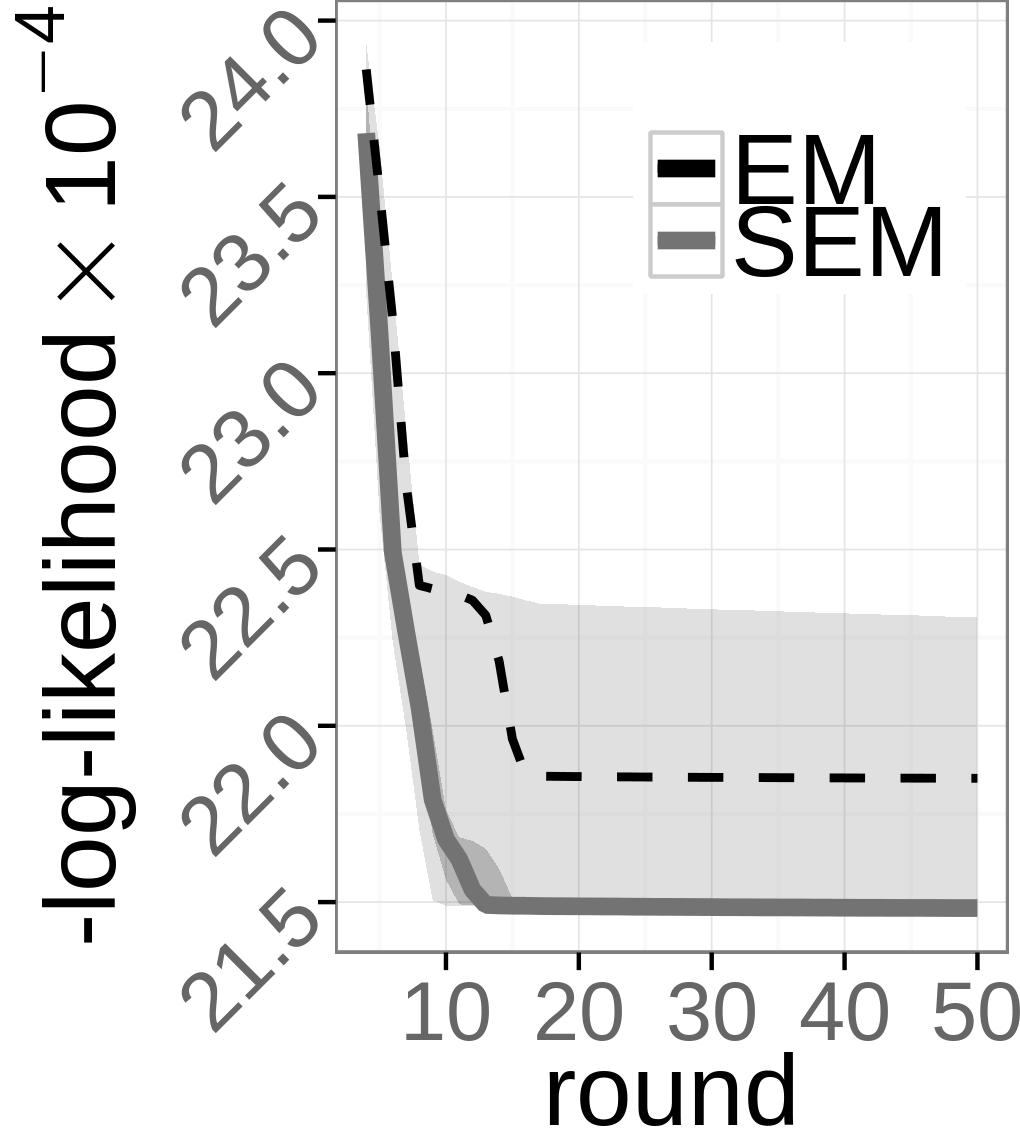}
}
\subfloat[$N=10\,000$]{
		\label{fig:gen10000_I19}		
		\includegraphics[width=0.23\textwidth]{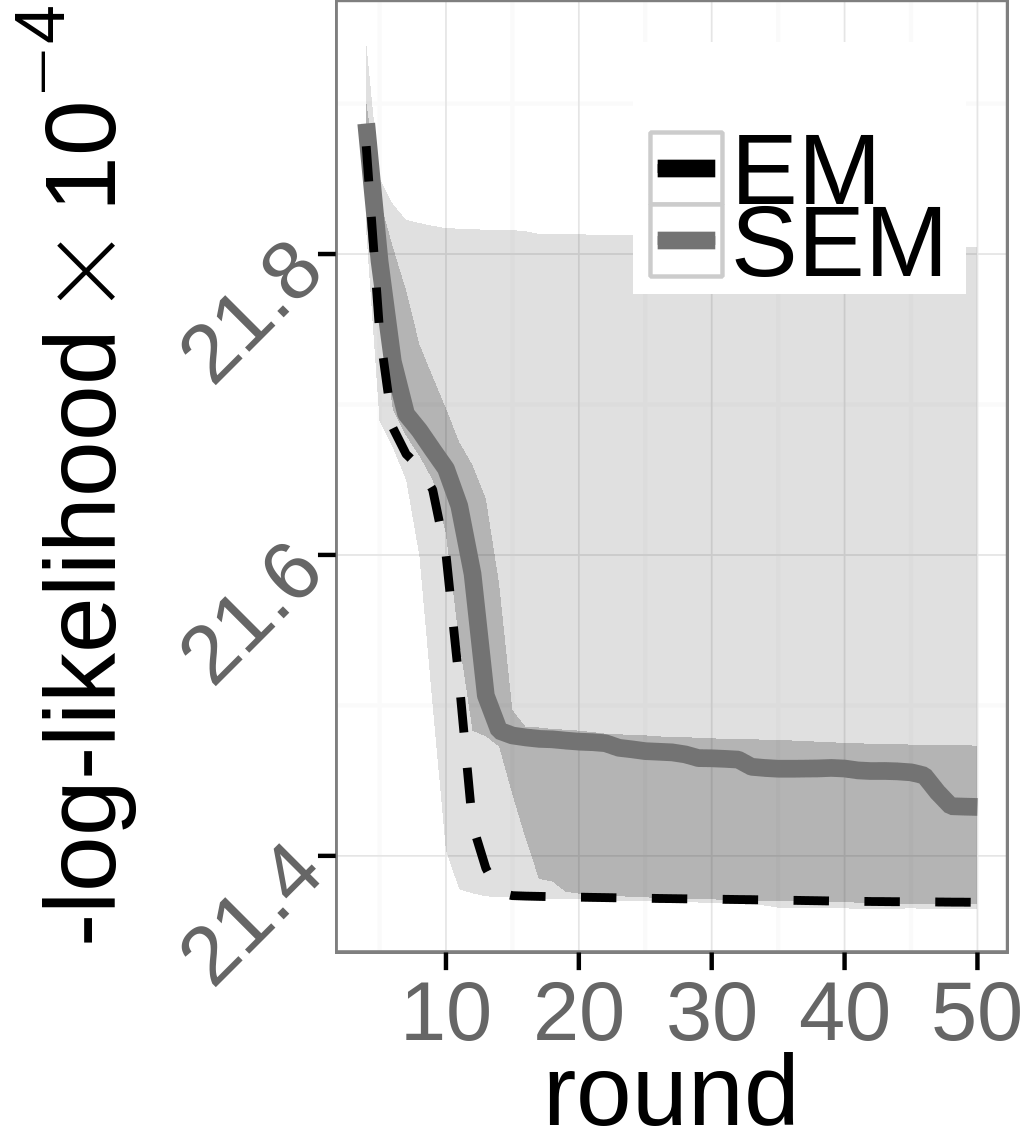}
}
\caption{Negative log-likelihood of parameters computed by the EM and SEM algorithm for the selected artificial data sets and different initial solutions.
Fig.~\ref{fig:gen10000_I25} to Fig.~\ref{fig:gen10000_I19} depict the results for the same artificial data set with $N=10\,000$ but different initial solutions. }
\label{fig:cost_gen}
\end{figure} 

\paragraph{Intermediate Solutions.} Some typical results of our second type of tests are shown in Fig.~\ref{fig:diff_aloihsb} to Fig.~\ref{fig:diff_gen10000}.
For the interpretation of the results, let $\Delta = \max_d \Delta_d$, $\Gamma_\mu = \sqrt{D}\Delta$, and $\Gamma_\Sigma = D\Delta^2$. Notice that the differences of the weights, means, and covariances are bounded by 1, $\Gamma_\mu$, and $\Gamma_\Sigma$, respectively.

For the normalized real world data sets $\Delta_d=1$. 
In relation to the spread, the parameter vectors are very similar.
For instance, regarding the ALOI data set, the weights, means and covariances differ by at most $0.3$, $0.02\cdot\Gamma_\mu$, and $0.0005\cdot\Gamma_\Sigma$, respectively.
In most cases, only in the first couple of rounds the difference between the parameters is larger.

For the artificial data sets, the spread in each dimension is approximately 40.
Regarding the data set with $N=1\,000\,000$, we observe the same behavior as for the real world data sets.
In some experiments we observe that the difference in the parameters of single components are substantially larger than those of the remaining components.
However, the differences are still small (cf. Fig.~\ref{fig:diff_gen1000000}).
That is, the weights, means and covariances differ by at most $0.0015$, $0.0015\cdot\Gamma_\mu$, and $0.0001\cdot\Gamma_\Sigma$, respectively.
Again, for small values of $N$ one can not expect good parameter estimates, which is indeed the case for data sets with $N=10\,000$.
In fact, the differences increase by a factor of at most $50$.
However, our first type of tests shows that this does not necessarily result in a different likelihood.

\begin{figure}
\begin{minipage}{.45\textwidth}
\centering
\subfloat[Difference $w_k^{EM}-w_k^{SEM}$ between the weights.]{
		\label{fig:aloi_k3_weight}
		\includegraphics[width=\textwidth]{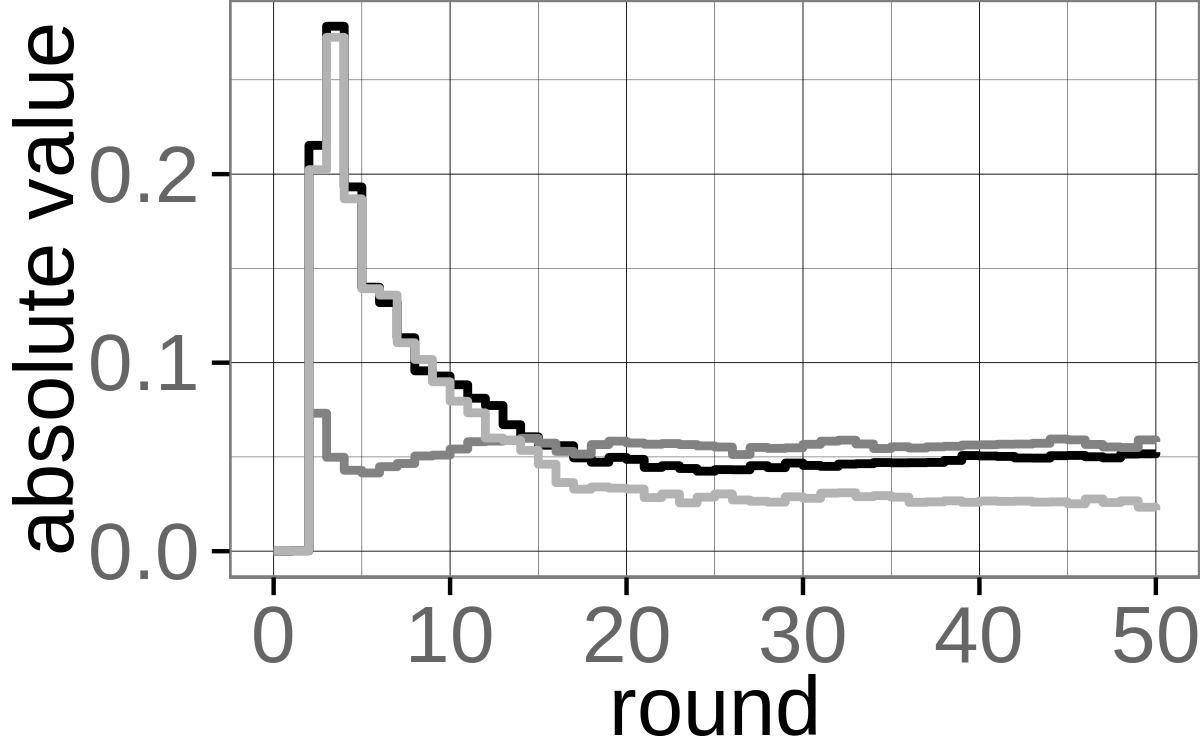}
}

\subfloat[Difference $\mu_k^{EM}-\mu_k^{SEM}$ between the means.]{
		\label{fig:aloi_k3_mean}
		\includegraphics[width=\textwidth]{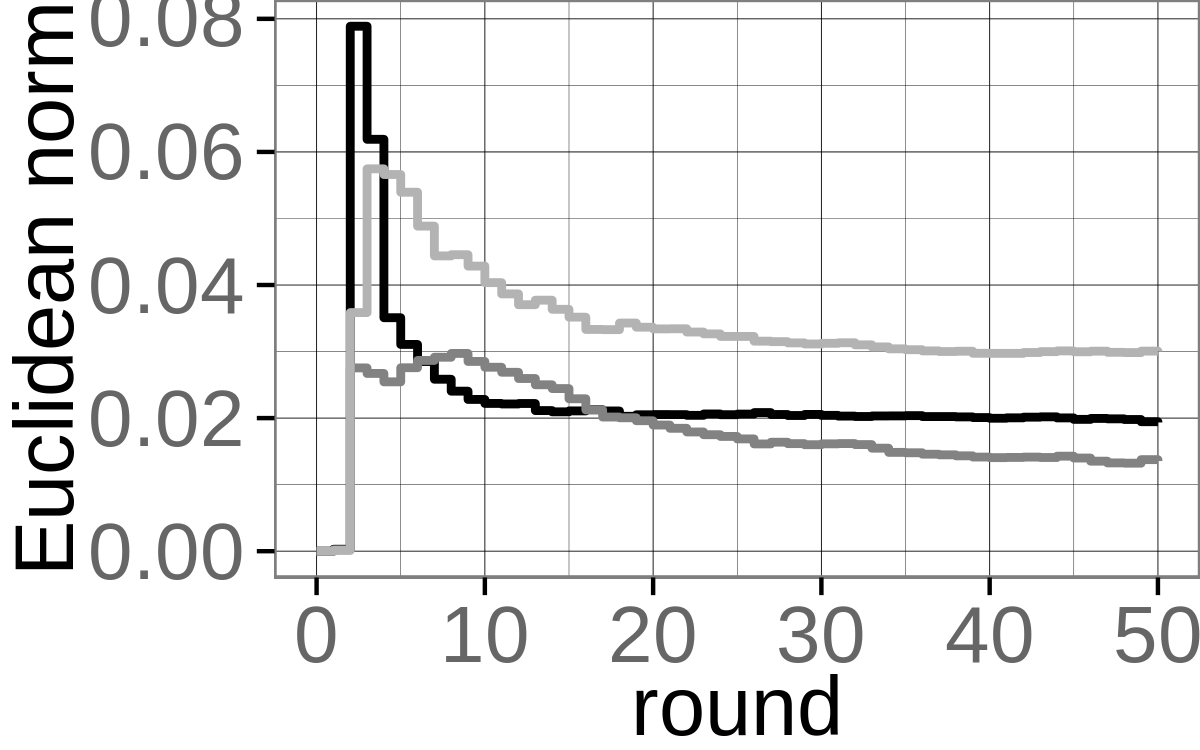}
}

\subfloat[Difference $\Sigma_k^{EM}-\Sigma_k^{SEM}$ between the covariances.]{
		\label{fig:aloi_k3_covar}
		\includegraphics[width=\textwidth]{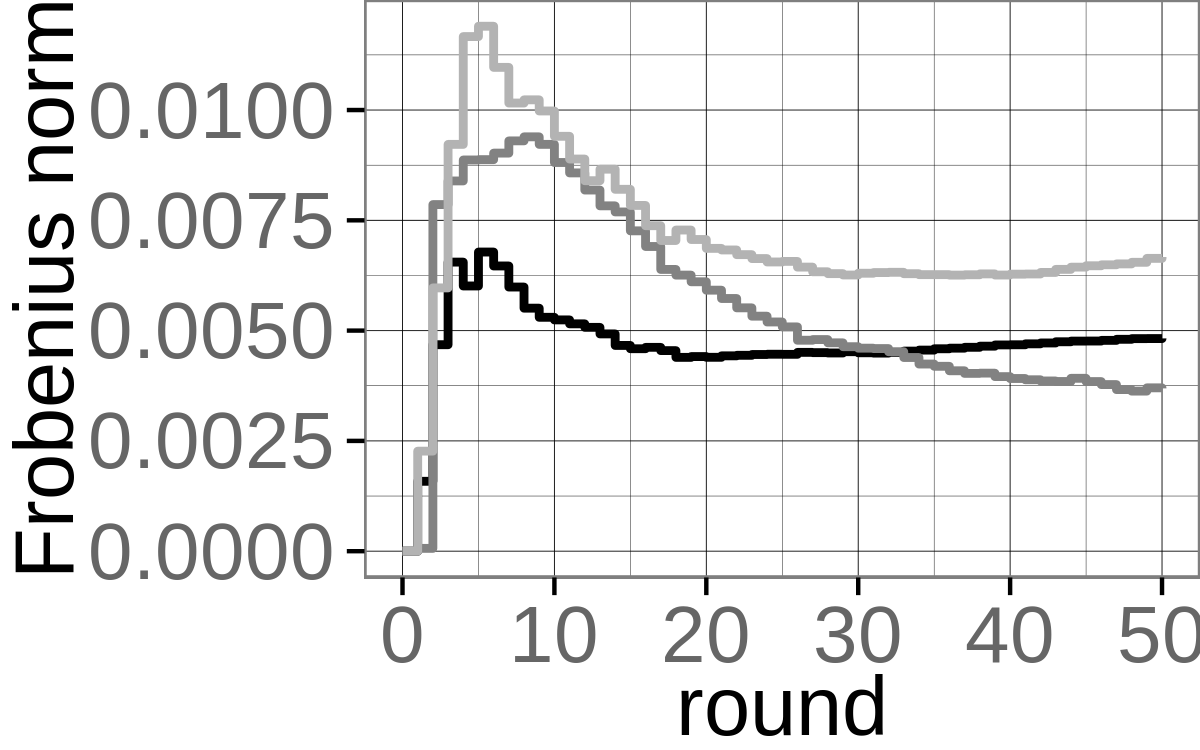}
}
\caption{Comparison of intermediate solutions of the EM and SEM algorithm given the \emph{Normalized ALOI} data set with $D=27$, $K=3$, $N=110\,250$, $\Delta_d=1$, $\Gamma_\mu\approx 5.2$, and $\Gamma_\Sigma= 27$.
All three figures depict the results for the same initial solution.
Each figure consists of three graphs, each depicting the difference for one component.
}
\label{fig:diff_aloihsb}
\end{minipage}
\qquad
\begin{minipage}{.45\textwidth}
\centering
\subfloat[Difference $w_k^{EM}-w_k^{SEM}$ between the weights.]{
		\label{fig:forest_k10_i3_weight}
		\includegraphics[width=\textwidth]{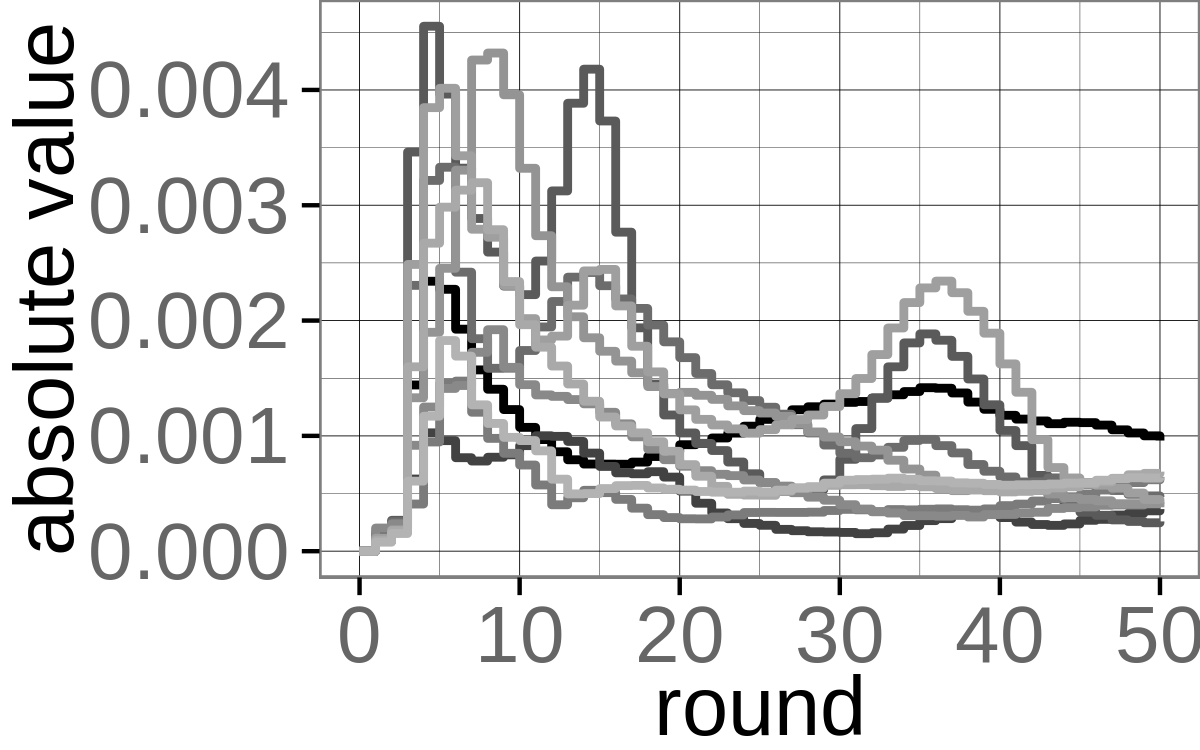}
}

\subfloat[Difference $\mu_k^{EM}-\mu_k^{SEM}$ between the means.]{
		\label{fig:forest_k10_i3_mean}
		\includegraphics[width=\textwidth]{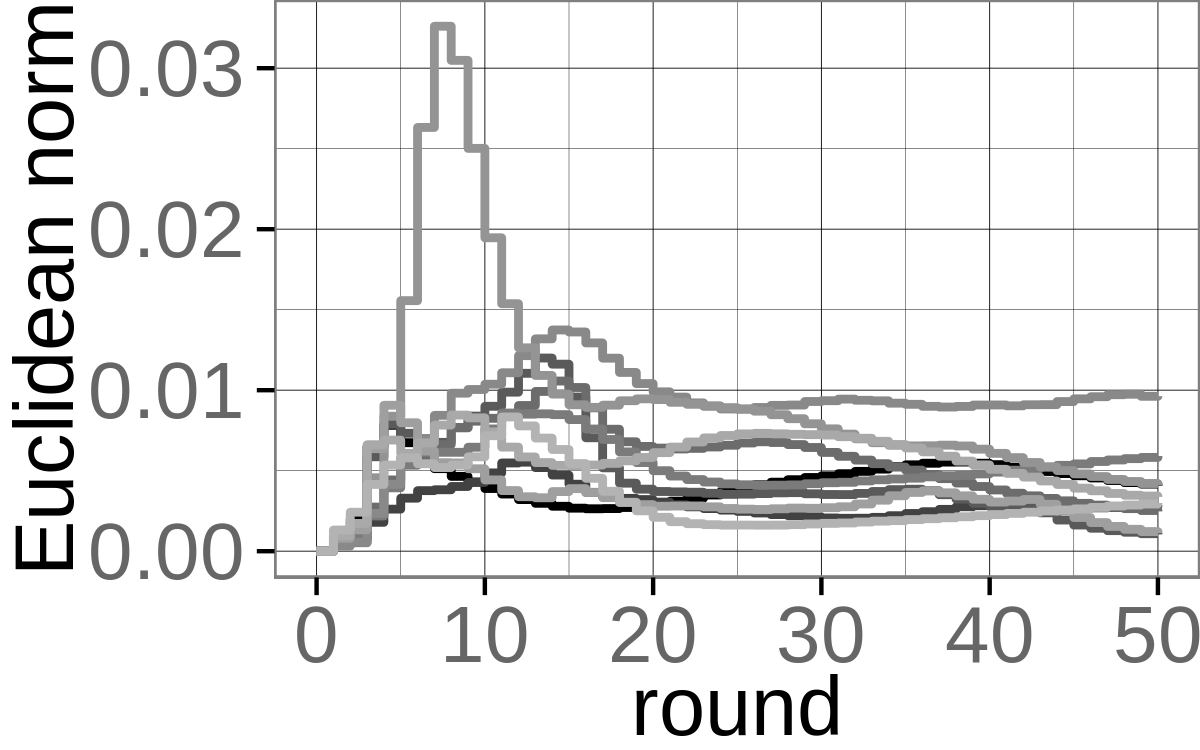}
}

\subfloat[Difference $\Sigma_k^{EM}-\Sigma_k^{SEM}$ between the covariances.]{
		\label{fig:forest_k10_i3_covar}
		\includegraphics[width=\textwidth]{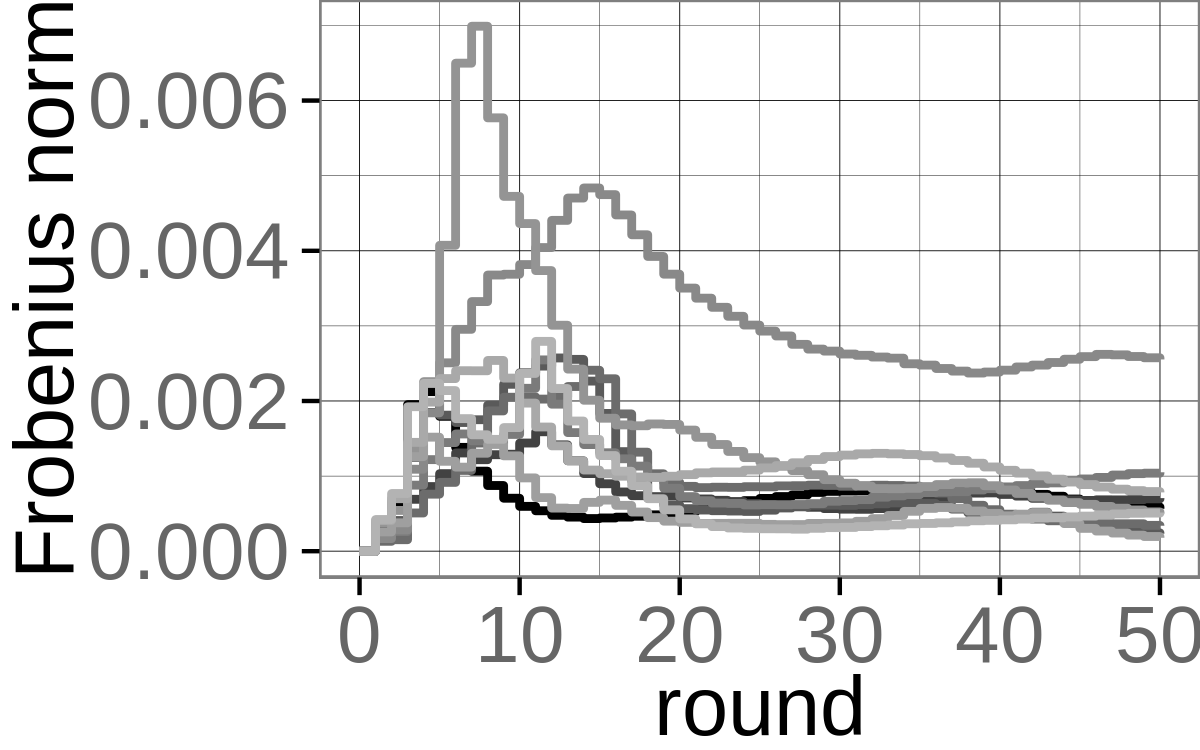}
}
\caption{Comparison of intermediate solutions of the EM and SEM algorithm given the \emph{Normalized Forest Covertype} data set with $D=K=10$, $N=581\,012$, $\Delta_d=1$, $\Gamma_\mu\approx 3.2 $, and $\Gamma_\Sigma = 10$.
All three figures depict the results for the same initial solution.
Each figure consists of ten graphs, each depicting the difference for one component.}
\label{fig:diff_forest}
\end{minipage}
\end{figure}

\begin{figure}
\begin{minipage}{.45\textwidth}
\centering
\subfloat[Difference $w_k^{EM}-w_k^{SEM}$ between the weights.]{
		\label{fig:gen_g3_w2_i2_weight}	
		\includegraphics[width=\textwidth]{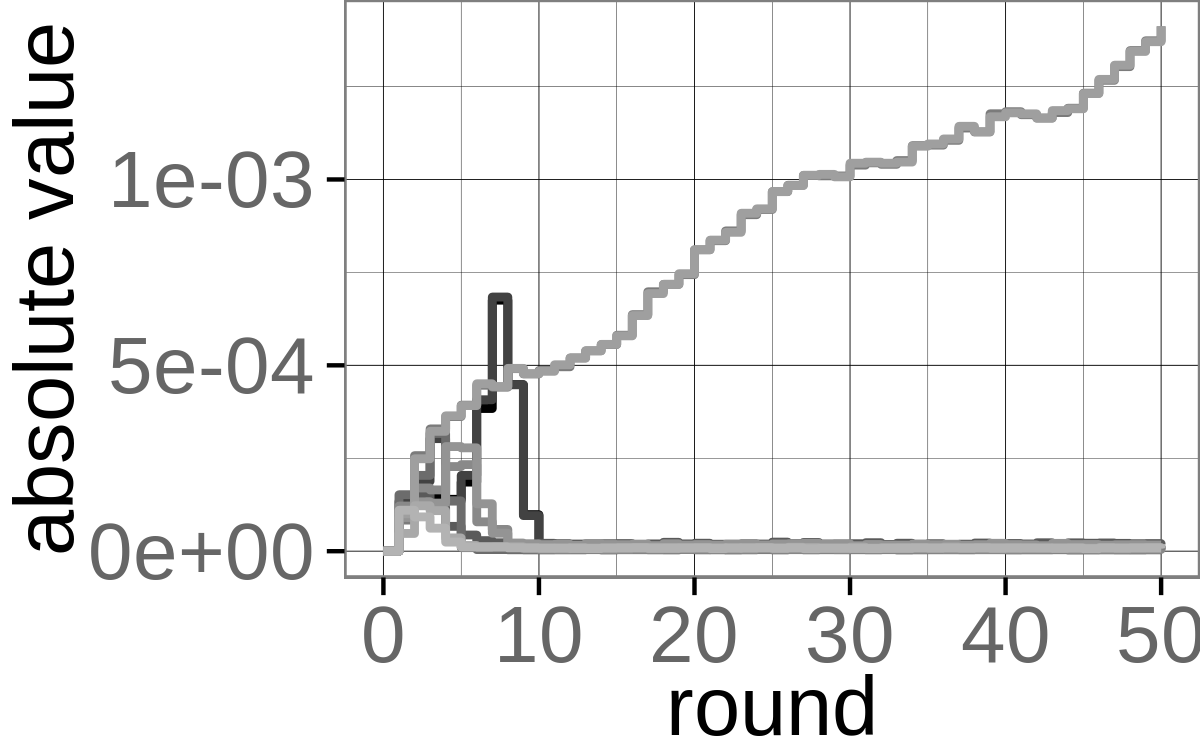}
}

\subfloat[Difference $\mu_k^{EM}-\mu_k^{SEM}$ between the means.]{
		\label{fig:gen_g3_w2_i2_mean}	
		\includegraphics[width=\textwidth]{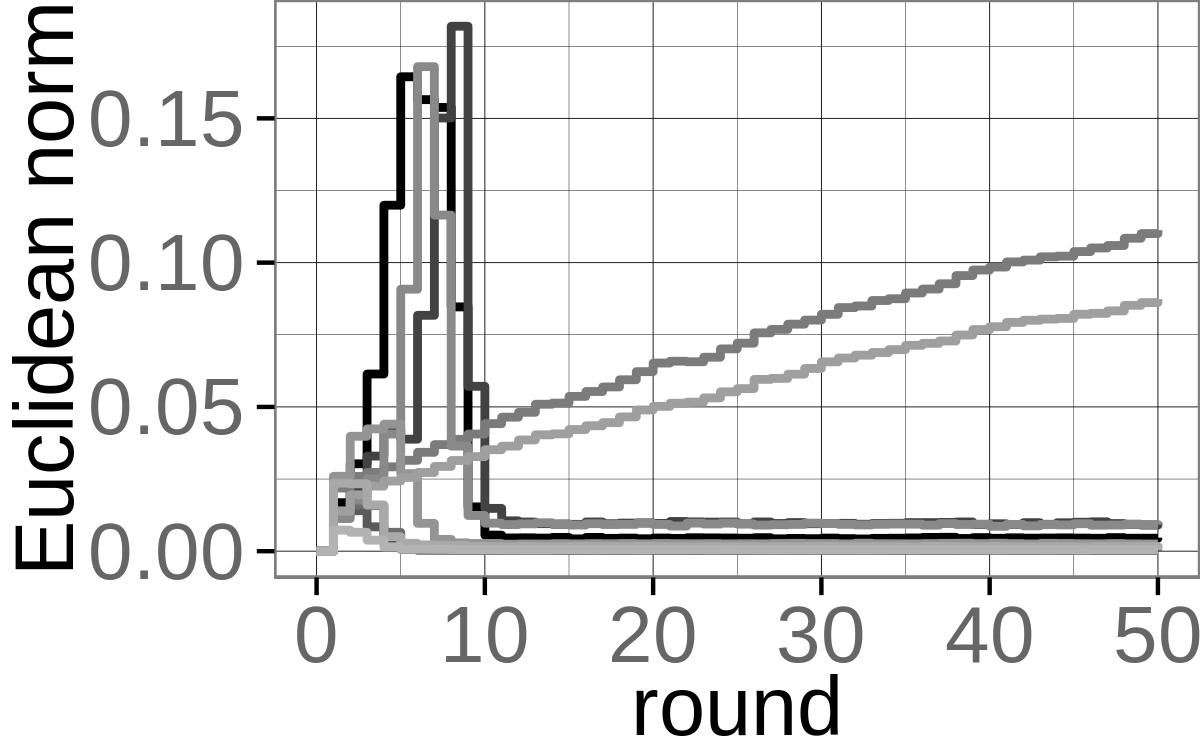}
}

\subfloat[Difference $\Sigma_k^{EM}-\Sigma_k^{SEM}$ between the covariances.]{
		\label{fig:gen_g3_w2_i2_covar}	
		\includegraphics[width=\textwidth]{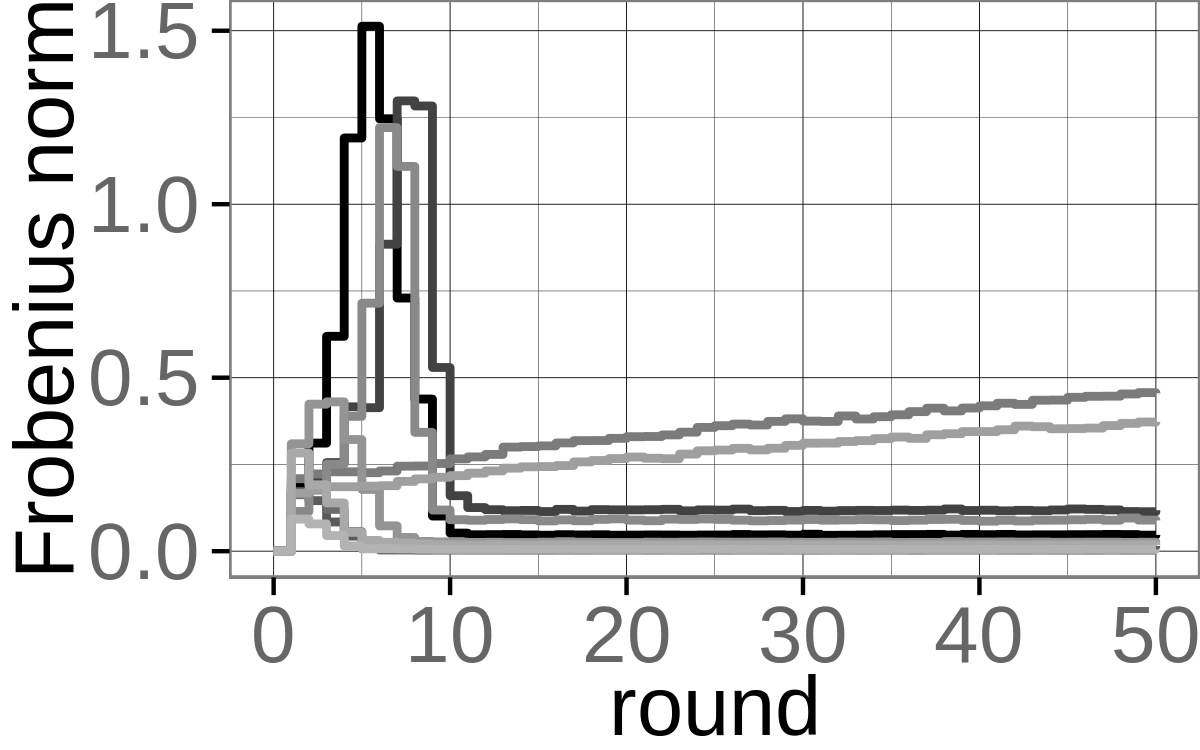}
}
\caption{Comparison of intermediate solutions of the EM and SEM algorithm given the selected \emph{artificial} data set with $D=K=10$, $N=1\,000\,000$, $\Delta_d=1$, $\Gamma_\mu\approx130$, and $\Gamma_\Sigma\approx16\,000$.
All three figures depict the results for the same initial solution.
Each figure consists of ten graphs, each depicting the difference for one component.}
\label{fig:diff_gen1000000}
\end{minipage}
\qquad
\begin{minipage}{.45\textwidth}
\centering
\subfloat[Difference $w_k^{EM}-w_k^{SEM}$ between the weights.]{
		\label{fig:gen10000_weight}	
		\includegraphics[width=\textwidth]{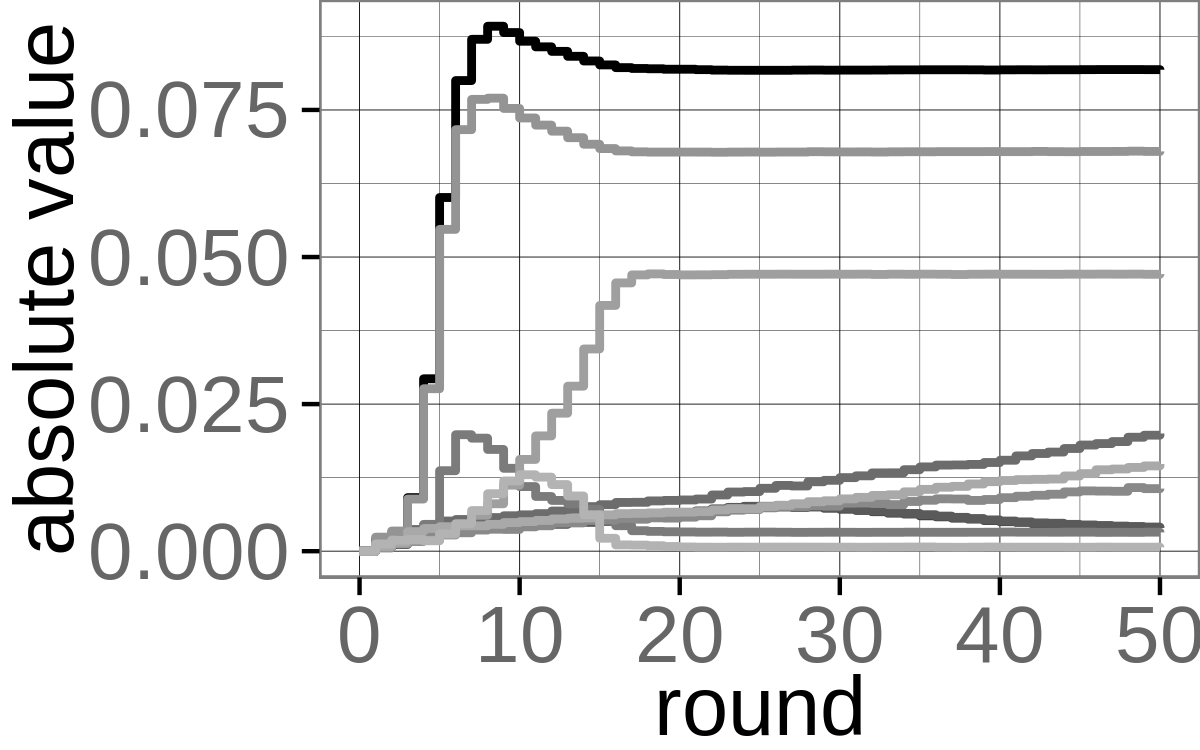}
}

\subfloat[Difference $\mu_k^{EM}-\mu_k^{SEM}$ between the means.]{
		\label{fig:gen10000_mean}	
		\includegraphics[width=\textwidth]{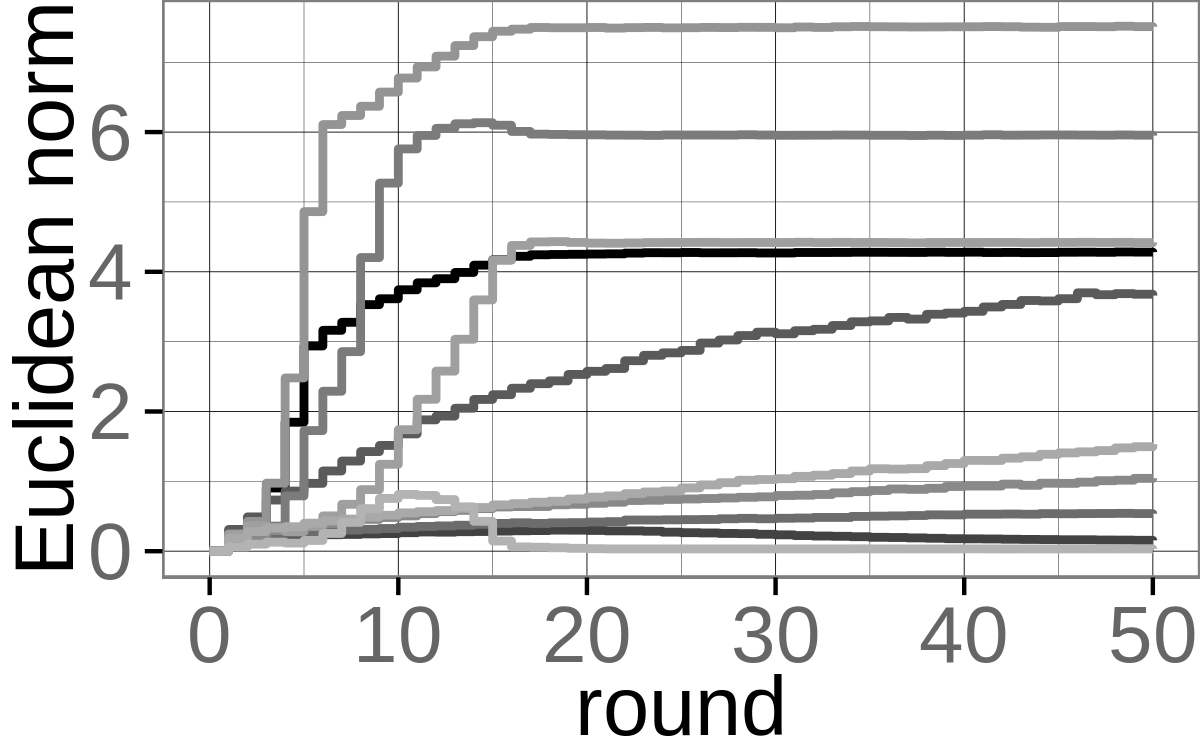}
}

\subfloat[Difference $\Sigma_k^{EM}-\Sigma_k^{SEM}$ between the covariances.]{
		\label{fig:gen10000_covar}	
		\includegraphics[width=\textwidth]{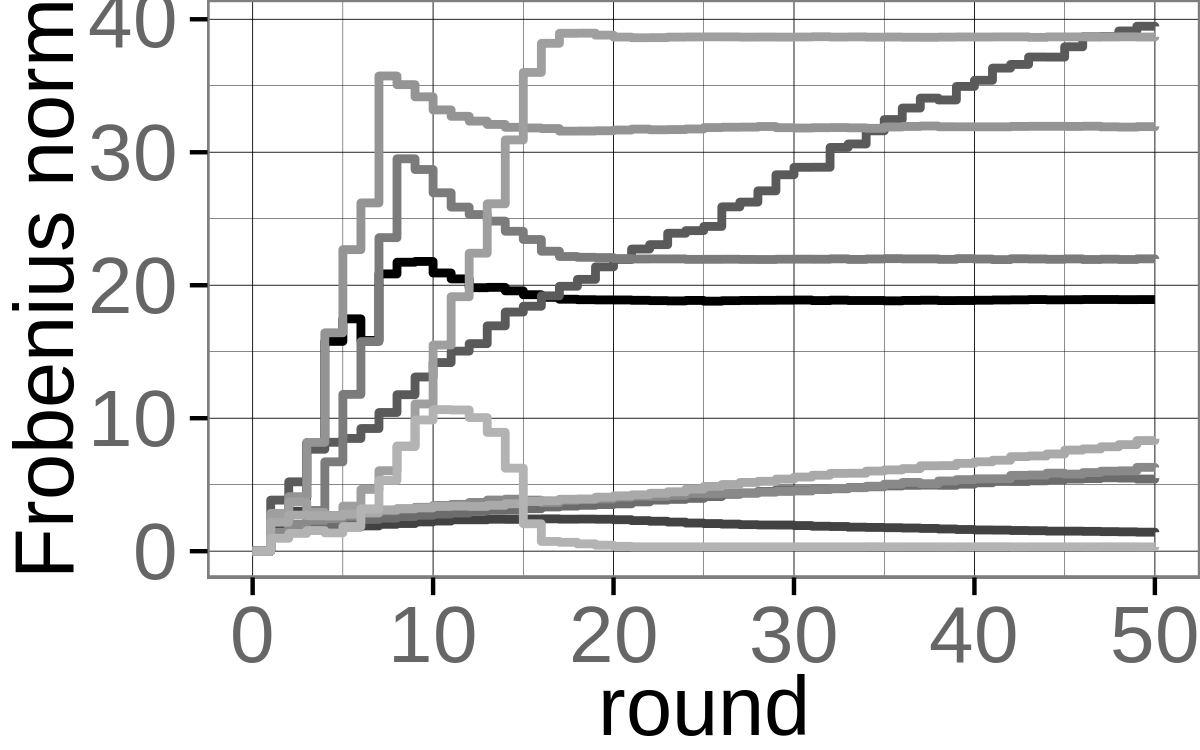}
}
\caption{Comparison of intermediate solutions of the EM and SEM algorithm given the selected \emph{artificial} data set with $D=K=10$, $N=10\,000$, $\Delta_d=1$, $\Gamma_\mu\approx130$, and $\Gamma_\Sigma\approx16\,000$ 
All three figures depict the results for the same initial solution.
Each figure consists of ten graphs, each depicting the difference for one component.}
\label{fig:diff_gen10000}
\end{minipage}
\end{figure}

\paragraph{Theoretical Bounds.} For our third type of tests, we depicted some characteristic results in Fig.~\ref{fig:theoretical_bounds}.
As in our first two types of tests, we expect worse results for the smaller artificial data set with $N=10\,000$.
Indeed, for the data set with $N=1\,000\,000$ our bound is smaller than $0.0054\cdot \Gamma_\mu$, while for $N=10\,000$ it is bounded by $0.065\cdot\Gamma_\mu$ (cf. Fig.~\ref{fig:gen10000_sem}). 
In Sec.~\ref{sem_gmm} we already discussed the limited applicability of our bounds for small responsibilities $r_k = N\cdot w^{EM}_k$.
Thus, the larger bounds observed in Fig.~\ref{fig:gen10000_sem} are not surprising, since smaller values of $N$ result in smaller values of $r_k = N\cdot w^{EM}_k$.

Furthermore, we get a bound up to $0.012\cdot\Gamma_\mu$ for the ALOI data set and $0.002\cdot\Gamma_\mu$ for the Forest data set.
The similar development of our bound and the actual difference indicates the accuracy of our analysis.
Moreover, we observe that the actual difference is significantly smaller than our bound, as it is to be expected.
The largest differences occur in the first couple of rounds matching our previous observations.

\begin{figure}
\centering
\subfloat[Normalized Forest Covertype data set ($K=10$).]{
		\label{fig:forest_k10_sem}		
		\includegraphics[width=0.45\textwidth]{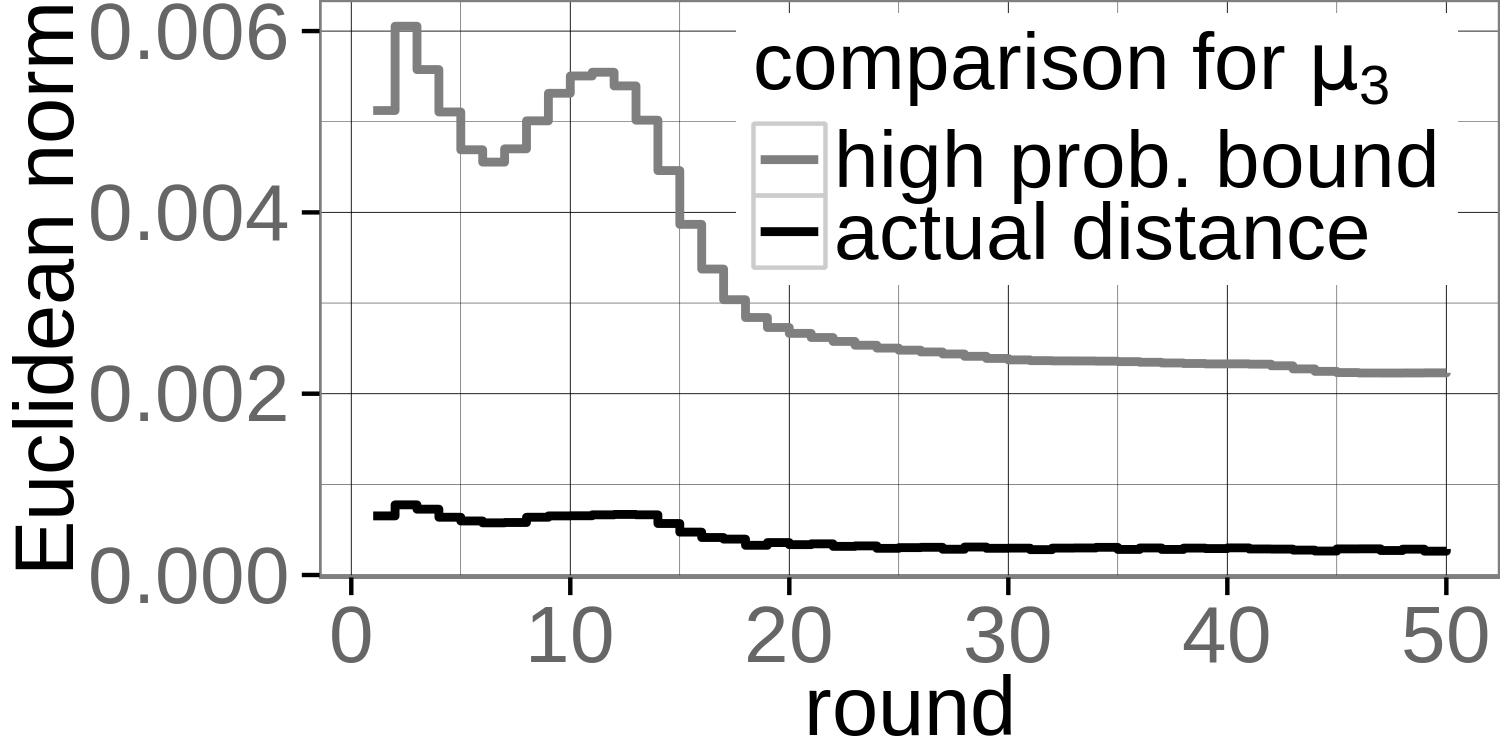}
}\qquad
\subfloat[$N=1\,000\,000$]{
		\label{fig:gen_g3_w2_i2_b1_sem}		
		\includegraphics[width=0.45\textwidth]{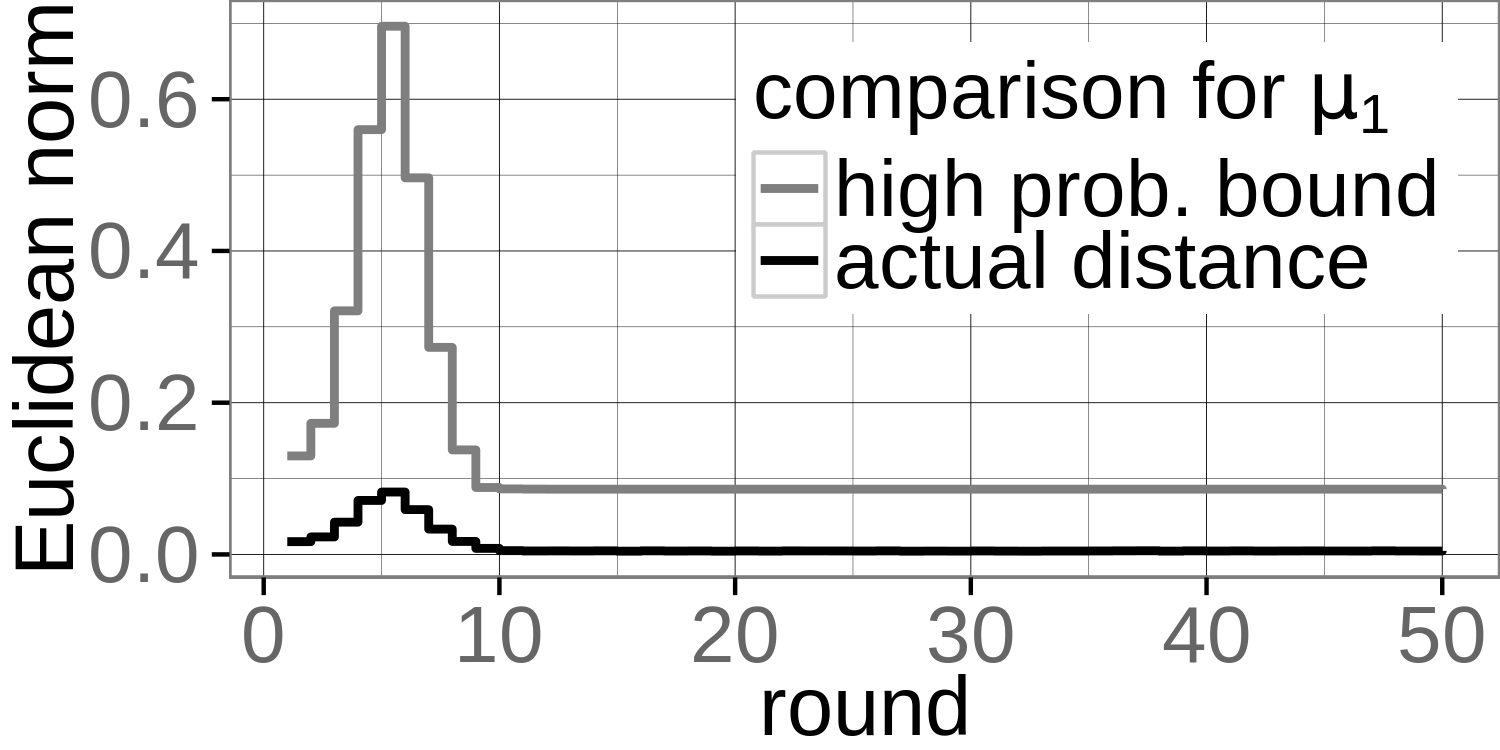}
}

\subfloat[Normalized ALOI data set ($K=3$).]{
		\label{fig:aloi_k3_sem}		
		\includegraphics[width=0.45\textwidth]{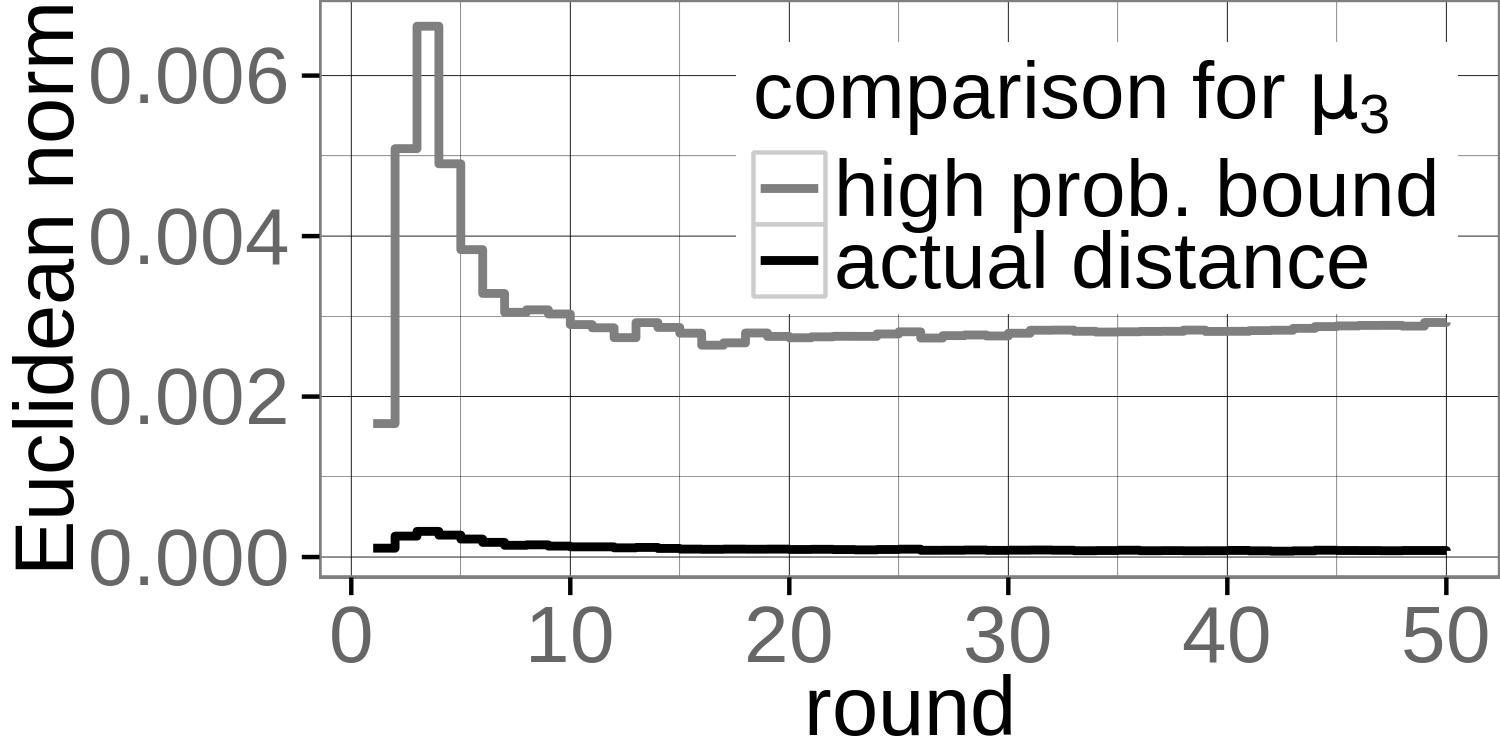}
}\qquad
\subfloat[$N=10\,000$]{
		\label{fig:gen10000_sem}		
		\includegraphics[width=0.45\textwidth]{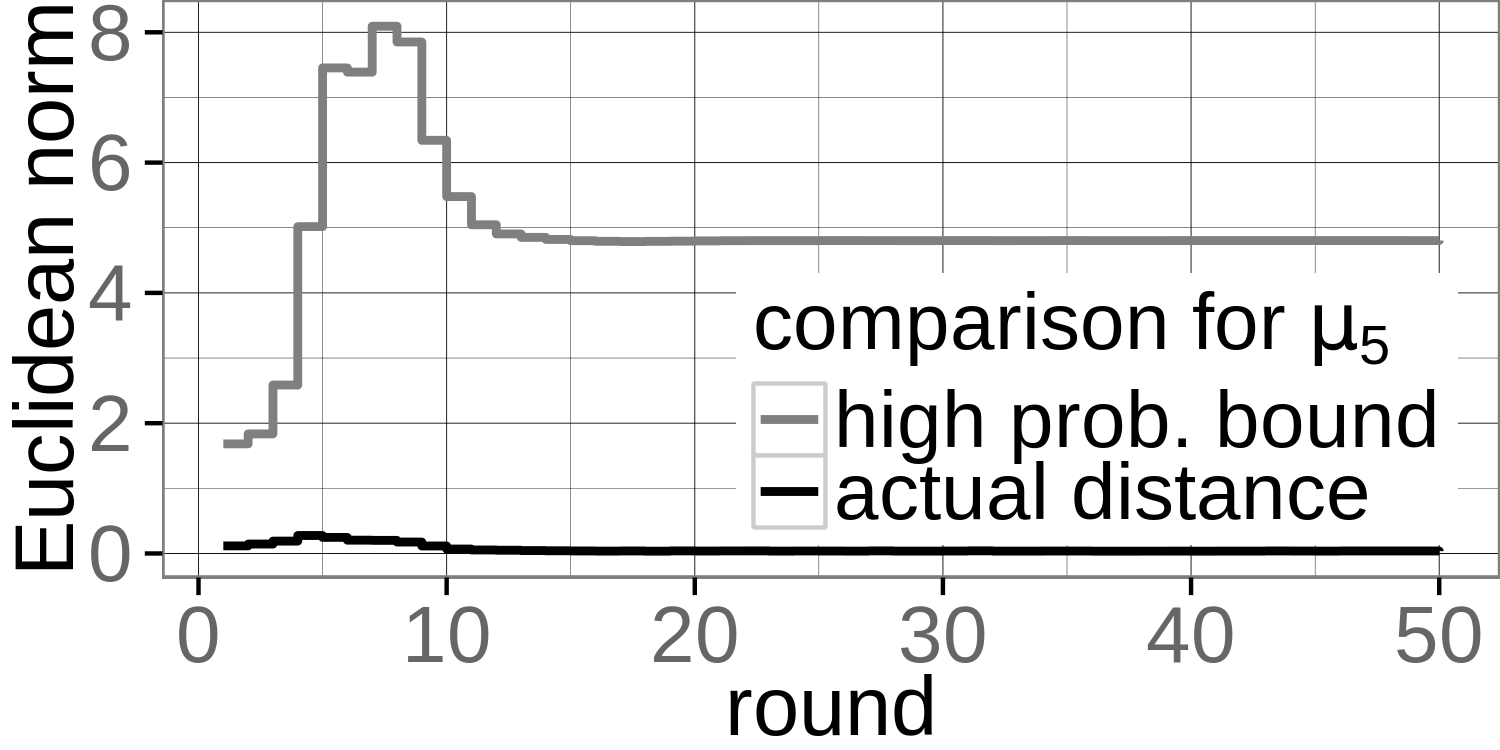}
}
\caption{Experimental difference and theoretical bound on the difference between the means for real world and selected artificial data sets and selected initial solutions.}
\label{fig:theoretical_bounds}
\end{figure}

\paragraph{Speedup.}
Additionally to the quality of the computed solutions, we also measured the running time of the algorithms.
For the SEM algorithm we observed a speedup of a factor 2 up to a factor 3 in all our experiments, as predicted in our running time analysis in Sec.~\ref{runtime}.

%
%
\bibliographystyle{splncs}
\bibliography{theoretical_and_experimental_comparison_of_the_em_and_sem_algorithm}

\begin{thebibliography}{10}

\bibitem{dempster77}
Dempster, A.P., Laird, N.M., Rubin, D.B.:
\newblock {Maximum likelihood from incomplete data via the EM algorithm}.
\newblock {Journal of the Royal Statistical Society, Series B: Statistical
  Methodology} \textbf{39}(1) (1977)  1--38

\bibitem{mclachlan08}
McLachlan, G.J., Krishnan, T.:
\newblock {The EM Algorithm and Extensions (Wiley Series in Probability and
  Statistics)}. 2 edn.
\newblock Wiley-Interscience (March 2008)

\bibitem{celeux85}
Celeux, G., Diebolt, J.:
\newblock {The SEM algorithm: a probabilistic teacher algorithm derived from
  the EM algorithm for the mixture problem}.
\newblock Computational Statistics Quarterly \textbf{2} (1985)  73--82

\bibitem{ip94}
{Ip, Eddie Hak-sing}:
\newblock {A} stochastic {EM} estimator in the presence of missing data --
  theory and applications.
\newblock PhD thesis, Stanford University (1994)

\bibitem{dias04}
Dias, J.G., Wedel, M.:
\newblock {An Empirical Comparison of EM, SEM and MCMC Performance for
  Problematic Gaussian Mixture Likelihoods}.
\newblock Statistics and Computing \textbf{14}(4) (October 2004)  323--332

\bibitem{bishop06}
Bishop, C.M.:
\newblock {Pattern Recognition and Machine Learning (Information Science and
  Statistics)}.
\newblock Springer-Verlag New York, Inc., Secaucus, NJ, USA (2006)

\bibitem{mitzenmacher05}
Mitzenmacher, M., Upfal, E.:
\newblock Probability and Computing: Randomized Algorithms and Probabilistic
  Analysis.
\newblock Cambridge University Press, New York, NY, USA (2005)

\bibitem{dasgupta07}
Dasgupta, S., Schulman, L.:
\newblock {A Probabilistic Analysis of EM for Mixtures of Separated, Spherical
  Gaussians}.
\newblock Journal of Machine Learning Research \textbf{8} (May 2007)  203--226

\bibitem{asuncion07}
A.~Asuncion, D.N.:
\newblock {UCI} machine learning repository (2007)

\bibitem{geusebroek05}
Geusebroek, J.M., Burghouts, G.J., Smeulders, A.W.M.:
\newblock The amsterdam library of object images.
\newblock International Journal of Computer Vision \textbf{61}(1) (2005)
  103--112

\bibitem{kriegel12}
Achtert, E., Goldhofer, S., Kriegel, H.P., Schubert, E., Zimek, A.:
\newblock Evaluation of clusterings -- metrics and visual support.
\newblock Data Engineering, International Conference on \textbf{0} (2012)
  1285--1288

\bibitem{kriegel11}
Kriegel, H.P., Schubert, E., Zimek, A.:
\newblock Evaluation of multiple clustering solutions.
\newblock In: Proc. ECML PKDD Workshop MultiClust. (2011)

\end{thebibliography}

\end{document}